\documentclass{article} 

\usepackage{xcolor}  
\usepackage[breaklinks=true,colorlinks]{hyperref}
\hypersetup{
  linkcolor={red!50!black},
  citecolor={blue!50!black},
}   
\usepackage{url}

\usepackage[utf8]{inputenc} 
\usepackage[truedimen,margin=20mm]{geometry} 

\usepackage[T1]{fontenc}    
\usepackage{url}            
\usepackage{booktabs}       
\usepackage{amsfonts}       
\usepackage{nicefrac}       
\usepackage{microtype}      
\usepackage[pdftex]{graphicx}
\usepackage{natbib} 
\usepackage{algorithm} 
\usepackage{algorithmic} 
\usepackage{multirow}
\usepackage{amsmath}
\usepackage{amssymb}
\usepackage{amsthm}
\usepackage{here}
\usepackage{wrapfig}
\usepackage{centernot}
\usepackage{enumitem}
\usepackage{comment}
\usepackage{subfig}
\usepackage{bbm}
\usepackage{titlesec}
\usepackage{abstract} 
\usepackage[capitalize,noabbrev]{cleveref}

\allowdisplaybreaks[1]
\titleformat*{\section}{\large\bfseries}

\newtheorem{theorem}{Theorem}
\newtheorem{lemma}{Lemma}
\newtheorem{proposition}{Proposition}
\newtheorem{definition}{Definition}
\newtheorem{corollary}{Corollary}
\newtheorem{assumption}{Assumption}
\newtheorem{example}{Example}
\theoremstyle{remark}
\newtheorem{remark}[theorem]{Remark}

\newcommand{\defeq}{\overset{\mathrm{def}}{=}}

\def\pd<#1>{\left\langle #1 \right\rangle}
\def\floor[#1]{\left\lfloor #1 \right\rfloor}
\def\ceil[#1]{\left\lceil #1 \right\rceil}

\newcommand{\rd}{\mathrm{d}}

\newcommand{\bB}{\mathbb{B}}
\newcommand{\bE}{\mathbb{E}}

\newcommand{\bR}{\mathbb{R}}

\newcommand{\cD}{\mathcal{D}}

\newcommand{\cL}{\mathcal{L}}
\newcommand{\cN}{\mathcal{N}}
\newcommand{\cP}{\mathcal{P}}
\newcommand{\cX}{\mathcal{X}}

\newcommand{\KL}{\mathrm{KL}}
\newcommand{\Ent}{\mathrm{Ent}}

\title{\Large Primal and Dual Analysis of Entropic Fictitious Play for Finite-sum Problems}

\author{Atsushi Nitanda$^{1,2}$, Kazusato Oko$^{3,2}$, Denny Wu$^{4,5}$, \\Nobuhito Takenouchi$^{1}$, Taiji Suzuki$^{3,2}$
\vspace{2mm}\\
\normalsize{\textit{$^1$Kyushu Institute of Technology}},
\normalsize{\textit{$^2$RIKEN Center for Advanced Intelligence Project}} \\
\normalsize{\textit{$^3$The University of Tokyo}},
\normalsize{\textit{$^4$The University of Toronto}},
\normalsize{\textit{$^5$Vector Institute for Artificial Intelligence}} \\
}

\date{}

\begin{document}
\twocolumn
\maketitle

\begin{abstract}
The entropic fictitious play (EFP) is a recently proposed algorithm that minimizes the sum of a convex functional and entropy in the space of measures --- such an objective naturally arises in the optimization of a two-layer neural network in the mean-field regime. In this work, we provide a concise primal-dual analysis of EFP in the setting where the learning problem exhibits a finite-sum structure. We establish quantitative global convergence guarantees for both the continuous-time and discrete-time dynamics based on properties of a \textit{proximal Gibbs measure} introduced in \citet{nitanda2022convex}. Furthermore, our primal-dual framework entails a memory-efficient particle-based implementation of the EFP update, and also suggests a connection to gradient boosting methods. We illustrate the efficiency of our novel implementation in experiments including neural network optimization and image synthesis. 
\end{abstract}
\section{Introduction} 

In this work we consider the optimization of an entropy-regularized convex functional in the space of measures: 
\begin{equation}
    \min_{\mu \in \cP_2} 
    \left\{ 
    F(\mu) + \lambda \Ent(\mu)
    \right\}.
    \label{eq:objective_intro}
\end{equation}
Note that when $F$ is a linear functional, i.e., $F(\mu) = \int f\rd\mu$, then the above objective admits a unique minimizer $\mu_*\propto\exp(-\lambda^{-1}f)$, samples from which can be obtained using various methods such as Langevin Monte Carlo. 
In the more general setting of nonlinear $F$, efficiently optimizing \eqref{eq:objective_intro} becomes more challenging, and existing algorithms typically take the form of an \textit{interacting-particle} update. 

Solving the objective \eqref{eq:objective_intro} is closely related to the optimization of neural networks in the \textit{mean-field regime}, which attracts attention because it captures the representation learning capacity of neural networks \citep{ghorbani2019limitations,li2020learning,abbe2022merged,ba2022high}. A key ingredient of the mean-field analysis is the connection between gradient descent on the wide neural network and the Wasserstein gradient flow in the space of measures, based on which the global convergence of training dynamics can be shown by exploiting the convexity of the loss function \citep{nitanda2017stochastic,mei2018mean,chizat2018global,rotskoff2018parameters,sirignano2020mean}. However, most of these existing analyses do not prove a convergence rate for the studied dynamics. 

The entropy term in \eqref{eq:objective_intro} leads to a \textit{noisy} gradient descent update, where the injected Gaussian noise encourages exploration \citep{mei2018mean,hu2019mean}. 
Recent works have shown that this added regularization entails exponential convergence of the continuous dynamics (termed the mean-field Langevin dynamics) under a \textit{logarithmic Sobolev inequality} condition that is easily verified in regularized empirical risk minimization problems with neural networks \citep{nitanda2022convex,chizat2022mean,chen2022uniform,suzuki2023uniformintime}. 
Moreover, novel update rules that optimize \eqref{eq:objective_intro} have also been proposed in \citet{nitanda2020particle,oko2022psdca,nishikawa2022}, for which quantitative convergence guarantees can be shown by adapting classical convex optimization theory into the space of measures. 

\begin{figure*}[t]
\begin{center}
\centering
\begin{minipage}[t]{0.68\linewidth} 
{\includegraphics[width=1\linewidth]{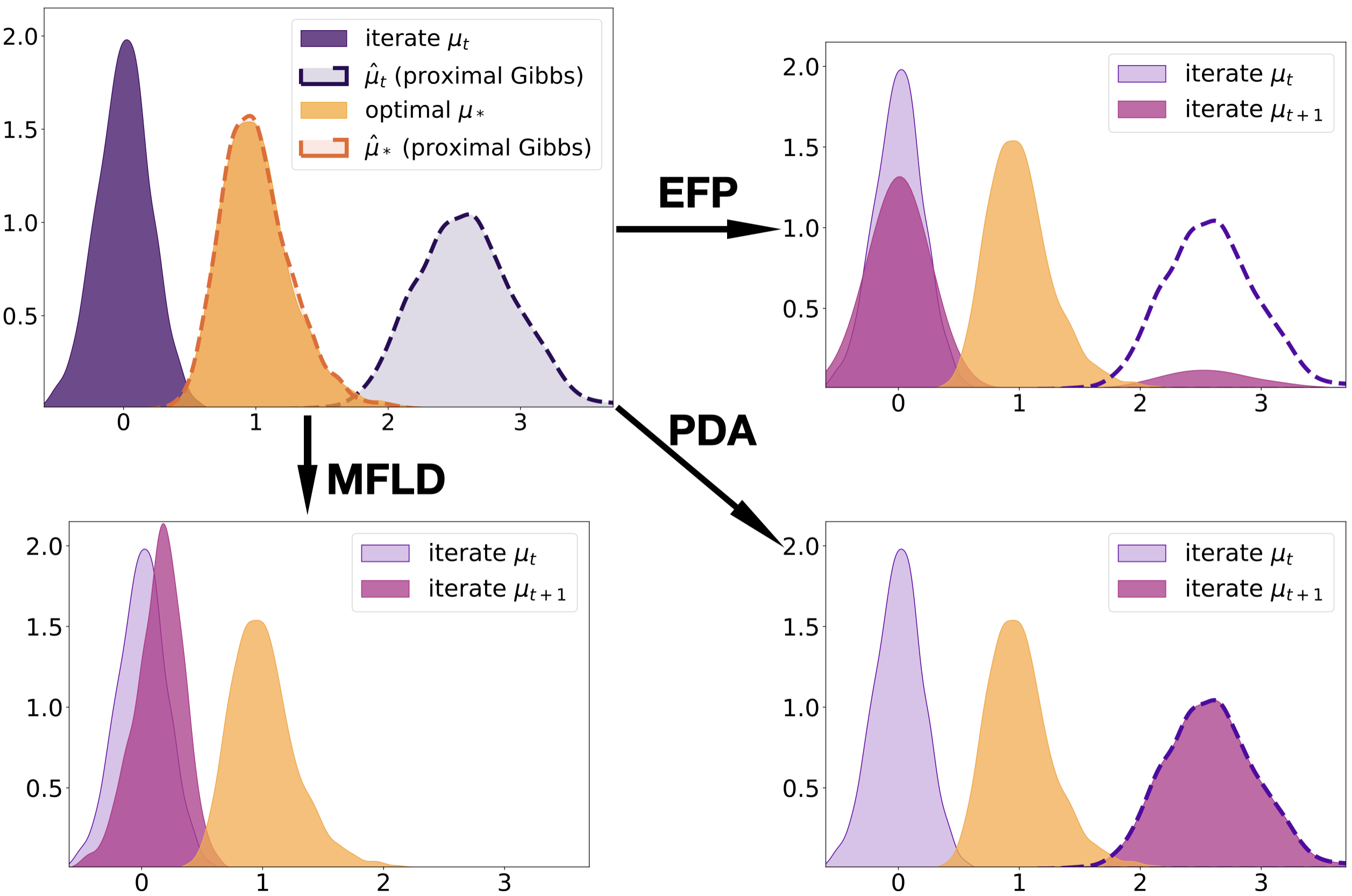}}  
\end{minipage}
 \vspace{-1.5mm}
\caption{\small 1D visualization of the parameter distribution of mean-field two-layer neural network (tanh) optimized by the mean-field Langevin dynamics (MFLD), particle dual averaging\protect\footnotemark (PDA), and entropic fictitious play (EFP), all of which can solve the self-consistent equation $\mu=\hat{\mu}$, where $\hat{\mu}$ is the \textit{proximal Gibbs measure} associated with $\mu$. We set $\lambda=\lambda'=10^{-2}$.
}
\label{fig:main} 
\end{center}
\vspace{-2mm}
\end{figure*} 

An important observation behind the design of these new algorithms is a \textit{self-consistent} condition of the global optimum \citep{hu2019mean}, namely, the optimal solution to the objective \eqref{eq:objective_intro} must satisfy
\begin{align}
\mu = \hat{\mu}, \quad\text{where~} \hat{\mu}(x) \propto \exp\left(-\frac{1}{\lambda}\frac{\delta F(\mu)}{\delta\mu}(x)\right), 
\label{eq:self-consistent}
\end{align}
in which $\frac{\delta F}{\delta\mu}$ denotes the first-variation of $F$. Following \citet{nitanda2022convex}, we refer to $\hat{\mu}$ as the \textit{proximal Gibbs distribution}. 
Based on properties of $\hat{\mu}$, \citet{oko2022psdca} and \citet{nitanda2022convex} established a primal-dual formulation of the optimization problem in the finite-sum setting.

In this work, we analyze a different update rule that optimizes \eqref{eq:objective_intro} recently proposed in \citet{ren2022entropic} termed the \textit{entropic fictitious play} (EFP).  
The EFP update is inspired by the classical fictitious play in game theory for learning Nash equilibria \citep{brown1951iterative}, which has been revisited by \cite{cardaliaguet2017learning,hadikhanloo2019finite,perrin2020fictitious,lavigne2022generalized} in the context of mean-field games. 
Intuitively, the EFP method successively accumulates $\hat{\mu} - \mu$ to the current iteration $\mu$ with an appropriate weight, as illustrated in Figure~\ref{fig:main}. In light of the self-consistent condition \eqref{eq:self-consistent}, this iteration is quite intuitive and natural, as it can be interpreted as a fixed-point iteration on the optimality condition $\mu = \hat{\mu}$. 

Despite its simple structure, EFP can be computationally prohibitive if implemented in a naive way. Specifically, for approximately optimizing mean-field models, we typically utilize a finite-particle discretization of the probability distribution. Then, the additive structure of EFP update requires additional particles to be added at each iteration. Therefore, the number of particles linearly increases in proportion to the number of steps, and nontrivial heuristics are needed to reduce the computational costs in general settings.

\footnotetext{For illustrative purposes, we simplify the PDA implementation from \citet{nitanda2020particle} by removing the weighted averaging in the gradient. }

\subsection{Our Contributions}
In this paper, we present a concise primal-dual analysis of EFP in the finite-sum setting (e.g., empirical risk minimization). 
We establish convergence in both continuous- and discrete-time of the primal and dual objectives, which is equivalent to the convergence of the Kullback-Leibler ($\KL$) divergence between $\mu$ and $\hat{\mu}$. Furthermore, our framework naturally suggests an efficient implementation of EFP for finite-sum problems, where the memory consumption remains constant across the iterations. 
Our contributions can be summarized as follows. 


\begin{itemize}[itemsep=0mm,leftmargin=*,topsep=0mm]
    \item We provide a simple analysis of EFP for finite-sum problems through the primal-dual framework. Our global convergence result holds in both continuous-and discrete-time settings (Theorem \ref{theorem:dual-convergence} and \ref{theorem:dual-convergence-discrete}). 
    \item Our theory suggests a memory-efficient implementation of EFP (Algorithm \ref{alg:implementable-efp}). Specifically, we do not memorize all particles from the previous iterations, but instead store the information via the  proximal Gibbs distribution, with memory cost that only depends on the number of training samples but not the number of iterations. 
    \item We present a connection between EFP and gradient-boosting methods, which enables us to establish discrete-time global convergence based on the Frank-Wolfe argument (Theorem \ref{theorem:fw-convergence}). In other words, in combination with the above algorithmic benefit in the finite-sum setting, EFP can be regarded as a memory-efficient version of the gradient-boosting method. 
    \item We employ our novel implementation of EFP in various experiments including neural network learning and image synthesis \citep{tian2022modern}.
\end{itemize}


\subsection{Notations}
$\|\cdot\|_2,~\|\cdot\|_\infty$ denote the Euclidean norm and uniform norm, respectively.
Given a probability distribution $\mu$ on $\bR^d$, we write the expectation w.r.t.~$\theta \sim \mu$ as $\bE_{\theta \sim \mu}[\cdot]$ or simply $\bE_{\mu}[\cdot]$, $\bE[\cdot]$ when the random variable and distribution are obvious from the context; e.g.~for a function $f: \bR^d \rightarrow \bR$, we write $\bE_\mu[f] = \int f(\theta) \rd \mu(\theta)$ when $f$ is integrable. 
$\KL$ stands for the Kullback-Leibler divergence:
$\KL(\mu\|\mu') \defeq \int \rd\mu(\theta) \log \frac{\rd \mu}{\rd \mu'}(\theta)$ and $\Ent$ stands for the negative entropy: $\Ent(\mu) = \int \rd\mu(\theta) \log \frac{\rd \mu}{\rd \theta}(\theta)$. 
Let $\cP_2$ be the space of probability distributions on $\bR^d$ such that absolutely continuous with respect to Lebesgue measure and the entropy and second moment are well-defined.

\section{Preliminaries}
\subsection{Problem Setup}
We say the functional $G:\cP_2 \to \bR$ is differentiable and convex when $G$ satisfies the following.
First, there exists a functional (referred to as a first variation): $\frac{\delta G}{\delta \mu}:~\cP_2 \times \bR^d \ni (\mu,\theta) \mapsto \frac{\delta G}{\delta \mu}(\mu)(\theta) \in \bR$ such that for any $\mu, \mu' \in \cP_2$,
\[ \left.\frac{\rd G (\mu+\epsilon (\mu'-\mu))}{\rd\epsilon} \right|_{\epsilon=0} 
= \int  \frac{\delta G}{\delta \mu}(\mu)(\theta) \rd(\mu'-\mu)(\theta) \]
and $G$ satisfies the convexity condition:
\begin{equation}\label{eq:convexity}
G(\mu') \geq G(\mu) + \int \frac{\delta G}{\delta \mu}(\mu)(\theta) \rd (\mu'-\mu)(\theta).      
\end{equation}

Let $F_0 :\cP_2 \to \bR$ be differentiable and convex, and define $F(\mu) = F_0(\mu) + \lambda'\bE_{\mu}[\|\theta\|_2^2]$.
We consider the minimization of an entropy-regularized convex functional:
\begin{equation}\label{prob:org}
    \min_{\mu \in \cP_2} 
    \left\{ 
    \cL(\mu) = F(\mu) + \lambda \Ent(\mu)
    \right\}.
\end{equation}

Note both $F$ and $\cL$ are also differentiable convex functionals.
Throughout the paper, we make the following regularity assumption on $F_0$.
\begin{assumption}\label{assumption:regularity}
We assume the first variation $\frac{\delta F_0}{\delta \mu}$ satisfies the Lipschitz continuity and boundedness in the following sense: there exists constant $M>0$ such that for any $\mu,\mu' \in \cP_2$ and for any $\theta,\theta' \in \bR^d$,
\begin{align*}
  \left| \frac{\delta F_0}{\delta \mu}(\mu)(\theta) - \frac{\delta F_0}{\delta \mu}(\mu)(\theta') \right| 
  &\leq M (W_2(\mu,\mu') + \|\theta - \theta'\|_2), \\
  \left| \frac{\delta F_0}{\delta \mu}(\mu)(\theta) \right| 
  &\leq M,
\end{align*}
where $W_2$ is the $2$-Wasserstein distance.
\end{assumption}

For $\mu \in \cP_2$, we introduce the {\it proximal Gibbs distribution}, which plays a key role in our analysis (for basic properties see \citet{nitanda2022convex,chizat2022mean}).
\begin{definition}[Proximal Gibbs distribution]
    We define $\hat{\mu}$ be the Gibbs distribution so that 
    \begin{equation}\label{eq:proximal-Gibbs}
        \frac{\rd\hat{\mu}}{\rd \theta}(\theta) = \frac{\exp\left( - \frac{1}{\lambda} \frac{\delta F(\mu)}{\delta \mu}(\theta) \right)}{Z(\mu)},
    \end{equation}
    where $Z(\mu)$ is the normalization constant and $\rd\hat{\mu} /\rd \theta$ represents the density function w.r.t.~Lebesgue measure.
\end{definition}

\remark{
Under Assumption \ref{assumption:regularity}, the existence and uniqueness of the minimizer $\mu_* \in \cP_2$ of $\cL$ is guaranteed by \citet{ren2022entropic}, and $\mu_*$ satisfies the first-order optimality condition: $\mu_* = \hat{\mu}_*$ \citep{hu2019mean}. 
}

We focus on the finite-sum setting as follows. Given differentiable convex functions: $\ell_i: \bR \rightarrow \bR$ and functions $h_i: \bR^d \rightarrow \bR$ ($i \in \{1,\ldots,n\}$), we consider the following minimization problem:
\begin{equation}\label{prob:erm-primal}
    \min_{\mu \in \cP_2} \left\{ \cL(\mu) = \frac{1}{n}\sum_{i=1}^n \ell_i( \bE[ h_i ] ) 
    + \lambda' \bE[ \|\theta\|_2^2 ] 
    + \lambda \Ent(\mu) \right\},   
\end{equation}
where the expectation is taken with respect to $\theta\sim \mu$, i.e., $\bE[ h_i ]=\bE_{\theta \sim \mu}[h_i(\theta)]$.
Note that this problem is a special case of (\ref{prob:org}) by setting $F_0(\mu)=\frac{1}{n}\sum_{i=1}^n \ell_i( \bE_{\mu}[h_i(\theta)] )$. 

\subsection{Example Applications}

One noticeable example of problem \eqref{prob:erm-primal} is the learning of mean-field neural networks via empirical risk minimization.
\begin{example}[Mean-field model]\label{example:mean-field-model}
Let $h(\theta,\cdot): \cX \rightarrow \bR$ be a function parameterized by $\theta \in \bR^d$, where $\cX$ is the data space.
The mean-field model is an integration of $h(\theta,\cdot)$ with respect to the distribution $\mu \in \cP_2$ over the parameter space: $\bE_{\theta \sim \mu}[ h(\theta,\cdot)]$. 
Given training examples $\{(x_i,y_i)\}_{i=1}^n \subset \cX \times \bR$, denote $h_i(\theta)$ as the output $h(\theta,x_i)$. Then we may choose the loss term as $\ell_i(\bE_{\mu}[h_i(\theta)]) = \ell( y_i, \bE[ h(\theta,x_i)])$, for convex loss functions such as the squared loss $\ell(y,y') = 0.5(y-y')^2$ and the logistic loss $\log(1+\exp(-yy'))$. 
\end{example}

Another example that falls into the framework of (\ref{prob:erm-primal}) is density estimation using a mixture model.
\begin{example}[Density Estimation]
Let $\mu * g_\delta$ ($\mu \in \cP_2$) be the Gaussian convolution of $\mu$:
\begin{equation*}
  \mu * g_\delta(\cdot) 
  \propto \int \exp\left( - \frac{\| \theta-\cdot\|_2^2}{2\sigma^2} \right) \rd \mu(\theta),
\end{equation*}
where $g_\delta(\theta) = (2\pi \sigma^2)^{-d/2} \exp( - \|\theta^2\|/(2\sigma^2))$.
Let $\{ \zeta_i \}_{i=1}^n \subset \bR^d$ be the i.i.d.~observations. Then, the log-likelihood $- \sum_{i=1}^n \log (\mu * g_\delta(\zeta_i))$ can be written as $\sum_{i=1}^n \ell_i(\bE_{\mu}[h_i])$, by setting $\ell_i = -\log$ and $h_i(\theta) = g_\delta(\theta - \zeta_i)$ since $\mu*g_\delta(\zeta_i) = \bE_{\theta \sim \mu}[g_\delta (\theta - \zeta_i)]$.
\end{example}

\section{Entropic Fictitious Play}
\subsection{The Ideal Update}
The entropic fictitious play (EFP) algorithm \citep{ren2022entropic} is the optimization method that minimizes the objective \eqref{prob:org}. 
The continuous evolution $\{\mu_t\}_{t\geq 0} \subset \cP_2$ of the EFP is defined as follows: for $\gamma > 0$,
\begin{equation}\label{eq:efp}
    \frac{\rd\mu_t}{\rd t} = \gamma ( \hat{\mu}_t - \mu_t).
\end{equation}
Under Assumption \ref{assumption:regularity}, the existence of differentiable density functions that solve (\ref{eq:efp}) is guaranteed. 

The time-discretized version of the EFP can be obtained by the explicit Euler's method. Given a step-size $\eta>0$,
\begin{equation}\label{eq:discrete-time-efp}
    \mu_{t+\eta} = (1-\eta\gamma)\mu_t + \eta \gamma \hat{\mu}_t.
\end{equation}
To execute this update, we need to compute the proximal Gibbs distribution $\hat{\mu}_t$, which can be accessed via standard sampling algorithms such as Langevin Monte Carlo. This ideal update is summarized in Algorithm \ref{alg:discrete-time-efp} below.
\begin{algorithm}[ht]
  \caption{Discrete-time Entropic Fictitious Play}
  \label{alg:discrete-time-efp}
\begin{algorithmic}
  \STATE {\bfseries Input:}
  $\mu^{(0)},~T,~\eta,~\gamma$
\vspace{1mm}
   \FOR{$t=0$ {\bfseries to} $T-1$}
   \STATE Compute $\hat{\mu}^{(t)}$ via standard sampling algorithms
   \STATE Update $\mu^{(t+1)} = (1-\eta\gamma) \mu^{(t)} + \eta \gamma \hat{\mu}^{(t)}$
   \ENDFOR
\vspace{1mm}   
\STATE Return $\mu^{(T)}$
\end{algorithmic}
\end{algorithm}

As previously remarked, in light of the first-order optimality condition $\mu_* = \hat{\mu}_*$, the discrete-time EFP (\ref{eq:discrete-time-efp}) is an intuitive update that mirrors the fixed-point iteration methods. 
However, the naive particle implementation of (\ref{eq:discrete-time-efp}) would be extremely expensive in terms of memory cost. 
In particular, the additive structure of the discrete-time EFP update requires us to store all the past particles along the trajectory. In other words, the memory cost scales linearly with the number of iterations. This prohibits the practical use of EFP beyond very short time horizon. 

\subsection{Efficient Particle Update for Finite-sum Settings}\label{subsec:algorithmic-advantage}
We now present an efficient implementation of discrete-time EFP by exploiting the finite-sum structure as in \citet{nitanda2020particle,oko2022psdca}. 
For comparison, we first provide a recap for the naive implementation.
To execute Algorithm \ref{alg:discrete-time-efp}, we run a sampling algorithm for each Gibbs distribution $\hat{\mu}^{(t)}$, and build an approximation $\hat{\mu}^{(t)} \sim \frac{1}{m} \sum_{r=1}^m \delta_{\theta_r^{(t)}}$ by finite particles $\theta_r^{(t)} \sim \hat{\mu}^{(t)}~(r=1,2,\ldots,m)$, where $\delta_\theta$ is the Dirac measure at $\theta \in \bR^d$. Then, the approximation $\underline{\mu}^{(t+1)}$ to $\mu^{(t+1)}$ can be recursively constructed as follows:
\begin{equation}\label{eq:additive-update}
 \underline{\mu}^{(t+1)} = (1-\eta\gamma) \underline{\mu}^{(t)} + \frac{\eta \gamma}{m}\sum_{r=1}^m \delta_{\theta_r^{(t)}}. 
\end{equation} 
It is clear that $\underline{\mu}^{(t)}$ consists of all particles obtained up to the $t$-th iteration, and thus the number of particles accumulates linearly with respect to the number of iterations.

\paragraph{Alternative implementation.} 
Our starting observation is that sampling algorithms for the Gibbs distribution such as Langevin Monte Carlo only require the computation of the gradient of the potential function. Recall that the potential function $\nabla \frac{\delta F (\mu^{(t)})}{\delta \mu}(\theta)$ of the proximal Gibbs distribution $\hat{\mu}^{(t)}$ for finite-sum problems (\ref{prob:erm-primal}) is written as follows: 
\begin{equation}\label{eq:prox-Gibbs-finite-sum}
    \frac{1}{n}\sum_{i=1}^n \ell'_i(\bE_{\mu^{(t)}}[h_i(\theta)]) \nabla h_i(\theta) + 2\lambda'\theta.
\end{equation}
Hence to approximate (\ref{eq:prox-Gibbs-finite-sum}), we do not need to store particles constituting $\underline{\mu}^{(t)}$ themselves, but only the empirical averages: $\bE_{\underline{\mu}^{(t)}}[h_i(\theta)]$~($i=1,2,\ldots,n$). 
Then, we notice that these empirical averages can be recursively computed in an online manner, because of the additive structure of the EFP update (\ref{eq:additive-update}): for $i=1,2,\ldots,n$,
\begin{equation*}\label{eq:addtive_average}
    \bE_{\underline{\mu}^{(t+1)}}[h_i(\theta)] 
    = (1-\eta \gamma)\bE_{\underline{\mu}^{(t)}}[h_i(\theta)] + \frac{\eta \gamma}{m} \sum_{r=1}^m h_i(\theta_r^{(t)}).
\end{equation*}
It is clear that the number of particles required to compute the above update is $m$ (i.e., the number of samples drawn from the current Gibbs distribution), which can be kept constant during training. This gives a memory-efficient and implementable EFP update, which is described in Algorithm \ref{alg:implementable-efp} using Langevin Monte Carlo (the sampling algorithm can be replaced by other alternatives). For simplicity, we denote $H^{(t)}_i = \bE_{\underline{\mu}^{(t)}}[h_i(\theta)]$ in the algorithm description.

\begin{algorithm}[ht]
  \caption{Efficient Implementation of EFP for Finite-sum Problem}
  \label{alg:implementable-efp}
\begin{algorithmic}
  \STATE {\bfseries Input:}
  $\mu^{(0)},~\lambda,~\lambda',~T,~S,~m,~\eta,~\gamma,~\eta'$
\vspace{1mm}
   \STATE Randomly sample $m$ particles $\{\theta_r^{(-1)}\}_{r=1}^{m}$ from $\mu^{(0)}$ 
   \STATE Compute initial empirical averages: for $i \in \{1,2,\ldots,n\}$   
   \STATE $H_i^{(0)}\leftarrow \frac{1}{m}\sum_{r=1}^m h_i(\theta_r^{(-1)})$
   \FOR{$t=0$ {\bfseries to} $T-1$}
   \STATE Initialize $\{\overline{\theta}_r^{(0)}\}_{r=1}^m$ (e.g., $\{\overline{\theta}_r\}_{r=1}^m \leftarrow \{\theta_r^{(t-1)}\}_{r=1}^m$)
   \FOR{$s=0$ {\bfseries to} $S-1$}   
    \STATE Run Langevin Monte-Carlo: for $r \in \{1,2,\ldots,m\}$
    \STATE $g^{(s)}(\overline{\theta}_r^{(s)}) \leftarrow \frac{1}{n}\sum_{i=1}^n \ell'_i(H_i^{(t)}) \nabla h_i(\overline{\theta}_r^{(s)})$
    \STATE $\xi_r^{(s)} \sim \cN(0,I_d)$~(i.i.d. Gaussian noise)
    \STATE $\overline{\theta}_r^{(s+1)} \leftarrow (1-2\eta'\lambda')\overline{\theta}_r^{(s)} - \eta' g^{(s)}(\overline{\theta}_r^{(s)}) + \sqrt{2\eta' \lambda}\xi_r^{(s)}$
   \ENDFOR
   \STATE $\{\theta_r^{(t)}\}_{r=1}^m \leftarrow \{\overline{\theta}_r^{(S)}\}_{r=1}^m$
   \STATE Update empirical averages: for $i \in \{1,2,\ldots,n\}$
   \STATE $H_i^{(t+1)} \leftarrow (1-\eta \gamma) H_i^{(t)} + \frac{\eta \gamma}{m}\sum_{r=1}^m h_i(\theta_r^{(t)})$
   \ENDFOR
\vspace{1mm}   
\STATE Return $\{\theta_r^{(T-1)}\}_{r=1}^m$ and $\{H_i^{(T)}\}_{i=1}^n$
\end{algorithmic}
\end{algorithm}

\remark{
The iteration complexity of Langevin Monte Carlo has been extensively studied. In our problem setting, we may verify that the Gibbs measure $\hat{\mu}_t$ satisfies a logarithmic Sobolev inequality (LSI) via the Holley-Stroock perturbation argument \cite{holley1987logarithmic}. Therefore, convergence of LMC is guaranteed from prior results (e.g., \citet{vempala2019rapid}).

}

\paragraph{Benefits of the proximal Gibbs distribution.}
When making the prediction by mean-field models (Example \ref{example:mean-field-model}) in machine learning applications, we need to compute $\bE_{\mu}[h(\theta,x)]$ on unseen input $x \in \cX$ with an optimized distribution $\mu$. In the following section, we establish primal and dual convergence of EFP, which implies the convergence of the proximal Gibbs distribution $\hat{\mu}_t$ to the optimum. Therefore, we can utilize particles $\{\theta_r^{(T-1)}\}_{r=1}^m$ to construct an actual predictor $\frac{1}{m}\sum_{r=1}^m h(\theta_r^{(T-1)},x)$ because these particles follow the proximal Gibbs distribution which is approximately optimal. 
Moreover, if we want to reduce the number of particles, we can apply efficient compression methods for the Gibbs distribution such as the kernel quadrature \citep{bach2017equivalence} and herding \cite{chen2018stein} to further improve the particle complexity. 

\section{Convergence in Continuous Time}\label{sec:continuous}
In this section, we provide a concise convergence analysis of the EFP in continuous-time \eqref{eq:efp} using properties of the proximal Gibbs distribution \eqref{eq:proximal-Gibbs}. 

\subsection{Primal Convergence}
Let $\mu_* \in \cP_2$ be an optimal solution of \eqref{prob:org}. The following property relates the suboptimality gap to the KL divergence involving the proximal Gibbs measure. 
\begin{proposition}[\citet{nitanda2022convex,chizat2022mean}]\label{prop:optimization_gap_for_strong_convex_problems}~
For any $\mu \in \cP_2$, we get
\begin{enumerate}
    \item[(1)] $\displaystyle  \frac{\delta \cL}{\delta \mu}(\mu) = \lambda\frac{\delta}{\delta \mu'} \KL(\mu' \| \hat{\mu})|_{\mu'=\mu} =  \lambda \log \frac{\rd \mu}{ \rd \hat{\mu}}$,     
    \item[(2)] $\displaystyle \lambda \KL( \mu \| \hat{\mu}) \geq \cL(\mu) - \cL(\mu_*) \geq \lambda \KL( \mu \| \mu_*)$.
\end{enumerate}
\end{proposition}

We now show that the continuous-time EFP \eqref{eq:efp} converges linearly to the global optimal solution $\mu_*$. 
\begin{theorem}\label{theorem:primal_convergence}
Let $\{\mu_t\}_{t \geq 0}$ be the evolution described by \eqref{eq:efp}.
Under Assumption \ref{assumption:regularity},
we get for $t \geq 0$,
\[ \cL(\mu_t) - \cL(\mu_*) \leq \exp( - \gamma t)( \cL(\mu_0) - \cL(\mu_*) ). \]
\end{theorem}
\begin{proof}
From Proposition \ref{prop:optimization_gap_for_strong_convex_problems}, we have
\begin{align*}
    \frac{\rd }{\rd t}&( \cL(\mu_t) - \cL(\mu_*) ) \\
    &= \int \frac{\delta \cL}{\delta \mu}(\mu_t)(\theta)\frac{\rd \mu_t}{\rd t}(\theta) \\
    &= \int \lambda\gamma \log \frac{\rd\mu_t}{\rd \hat{\mu}_t}(\theta) \rd (\hat{\mu}_t - \mu_t)(\theta) \\    
    &= -\lambda \gamma \left( \KL( \mu_t \| \hat{\mu}_t ) + \KL( \hat{\mu}_t \| \mu_t ) \right) \\
    &\leq -\gamma (\cL(\mu_t) - \cL(\mu_*)).
\end{align*}
The statement then follows from a straightforward application of Grönwall’s inequality. 
\end{proof}

\subsection{Primal-Dual Convergence}
To introduce a primal-dual formulation of the finite-sum minimization problems \eqref{prob:erm-primal}, we denote $\ell_i^*(\cdot)$ as the Fenchel conjugate, i.e.,
\[ \ell_i^*(z^*) = \sup_{z \in \bR } \{ zz^* - \ell_i(z) \}. ~~\text{for } z^* \in \bR\]
Also, for any given vector $g=\{g_i\}_{i=1}^n \in \bR^n$, we define
\begin{equation*}
    q_g(\theta) 
    = \exp\left( - \frac{1}{\lambda} \left( \frac{1}{n}\sum_{i=1}^n h_{i}(\theta)g_i 
    + \lambda'\|\theta\|_2^2 \right) \right).
\end{equation*}
Then, the dual problem of (\ref{prob:erm-primal}) is defined as 
\begin{equation}\label{prob:erm-dual}
    \max_{g \in \bR^n}\left\{ 
    \cD(g) = -\frac{1}{n}\sum_{i=1}^n \ell_i^*(g_i) 
    - \lambda \log\int q_g(\theta) \rd\theta
    \right\}.
\end{equation}
The duality theorem \citep{rockafellar,bauschke2011convex,oko2022psdca} guarantees the relationship $\cD(g) \leq \cL(\mu)$ for any $g \in \bR^n$ and $\mu \in \cP_2$ , and it is known that the duality gap $\cL(\mu) - \cD(g)$ vanishes at the solutions of (\ref{prob:erm-primal}) and (\ref{prob:erm-dual}) when they exist.
In our problem setting, a more precise result is given in \citet{nitanda2022convex}.
We write $g_\mu = \{ \ell'_i( \bE_{\mu}[h_i(\theta)] ) \}_{i=1}^n \in \bR^n$ ($\mu \in \cP_2$). 
Note that $q_g$ and $g_\mu$ are connected to the proximal Gibbs distribution through the relation: $q_{g_\mu} \propto \hat{\mu}$.
The following theorem from \citet{nitanda2022convex} exactly characterizes the duality gap $\cL(\mu) - \cD( g_\mu )$ between $\mu \in \cP_2$ and $g_\mu \in \bR^n$ via the KL divergence to the proximal Gibbs distribution $\hat{\mu}$. 

\begin{theorem}[Duality Theorem \citep{nitanda2022convex}]\label{theorem:duality}
Suppose $\ell_i(\cdot)$ is convex and differentiable.
For any $\mu \in \cP_2$, the duality gap between $\mu$ and $g_\mu$ satisfies
\[ 0 \leq \cL(\mu) - \cD(g_\mu) = \lambda \KL( \mu \| \hat{\mu} ). \]
\end{theorem}

For the rest of the analysis, we further assume $\ell_i$ is the twice differentiable convex function. 
The following theorem shows the exponential convergence of $\KL(\mu_t \| \hat{\mu}_t)$ which upper bounds the duality gap. 

\begin{theorem}\label{theorem:dual-convergence}
    Let $\{\mu_t\}_{t \geq 0}$ be the evolution described by \eqref{eq:efp}.
    Under Assumptions \ref{assumption:regularity}, we get
    \begin{equation*}
    \KL(\mu_t \| \hat{\mu}_t) \leq \exp(-\gamma t) \KL(\mu_0 \| \hat{\mu}_0).
    \end{equation*}
\end{theorem}
\begin{proof}
We calculate the time derivative of the primal and dual objectives.
On one hand, following the same computation as the proof of Theorem \ref{theorem:primal_convergence},
\begin{align*}
    \frac{\rd }{\rd t} \cL(\mu_t) \leq -\lambda \gamma \KL( \mu_t \| \hat{\mu}_t ).
\end{align*}
On the other hand, we have
\begin{align*}
    &-\frac{\rd }{\rd t} \cD(g_{\mu_t}) 
    = -\nabla \cD(g_{\mu_t})^\top \frac{\rd g_{\mu_t}}{\rd t} \\
    &= \frac{1}{n}\sum_{i=1}^n \left( \ell_i^{*\prime}(g_{\mu_t,i}) - \frac{1}{\int q_{g_{\mu_t}}(\theta)\rd\theta} \int h_i(\theta) q_{g_{\mu_t}}(\theta)\rd\theta \right) \\
    &\quad \cdot \frac{\rd}{\rd t} \ell_i'( \bE_{\mu_t}[h_i(\theta)]) \\
    &= -\frac{\gamma}{n}\sum_{i=1}^n \left( \bE_{\mu_t}[h_i(\theta)] - \bE_{\hat{\mu}_t}[h_i(\theta)] \right)^2  \ell_i''( \bE_{\mu_t}[h_i]) 
    \leq 0,
\end{align*}
where we used $\hat{\mu}_t \propto q_{g_{\mu_t}}(\theta)\rd \theta$, $\ell_i^{*\prime} = (\ell_i')^{-1}$, and
\begin{align*}
\frac{\rd}{\rd t} \ell_i'( \bE_{\mu_t}[h_i(\theta)])
&= \ell_i''( \bE_{\mu_t}[h_i]) \int h_i(\theta) \frac{\rd \mu_t}{\rd t}(\theta) \\
&= \ell_i''( \bE_{\mu_t}[h_i]) (\bE_{\hat{\mu}_t}[h_i] - \bE_{\mu_t}[h_i] ).
\end{align*}

Combining these inequalities with Theorem \ref{theorem:dual-convergence},
\begin{align*}
    \frac{\rd}{\rd t} \KL(\mu_t \| \hat{\mu_t}) 
    &= \frac{1}{\lambda}\frac{\rd}{\rd t}( \cL(\mu_t) - \cD(g_{\mu_t})) \\
    &\leq -\gamma \KL( \mu_t \| \hat{\mu}_t ).
\end{align*}
Therefore, Grönwall’s inequality finishes the proof.
\end{proof}

Finally, Theorem \ref{theorem:dual-convergence} in combination with the following result entails the convergence of $\cL(\hat{\mu}_t)$ to the optimal value. 
\begin{theorem}[\citet{nitanda2022convex}]
Suppose Assumptions in Theorem \ref{theorem:duality} holds and suppose that $\ell_i$ is $L$-Lipschitz smooth, that is, $|\ell'(z) - \ell(z')| \leq L|z-z'|$ for any $z,z' \in \bR$, and $\|h_i\|_\infty \leq B$. Then,  
\[ \cL(\hat{\mu}_t) - \cD(g_{\mu_t}) \leq (\lambda + 2 B^2 L) \KL( \mu_t \| \hat{\mu}_t ). \]
\end{theorem}

\section{Convergence in Discrete Time}
Now we establish the convergence guarantee for the discrete-time EFP (Algorithm \ref{alg:discrete-time-efp}). We present two different proof strategies that might be of independent interest. 
The first approach uses a one-step interpolation argument that builds upon the continuous-time result in Section \ref{sec:continuous}; whereas the second approach is a Frank-Wolfe analysis based on the gradient-boosting interpretation of EFP.
\subsection{Analysis via One-step Interpolation}
The key to our first convergence proof is the one-step interpolation argument. For Algorithm \ref{alg:discrete-time-efp}, we take the following continuous linear interpolation between $\mu^{(t)}$ and $\mu^{(t+1)}$ over the space of probability distributions. Let $\nu_0 = \mu^{(t)}$, and 
\[ \nu_{s} = \nu_0 + \gamma s(\hat{\nu}_0 -\nu_0) = \mu^{(t)} + s\gamma (\hat{\mu}^{(t)} -\mu^{(t)}), \]
where $s \in [0,\eta]$.
Then, $\nu_{s}$ is a linear approximation to the continuous EFP (\ref{eq:efp}) and we obtain the following:
\begin{align*}
 \frac{\rd }{\rd s}\cL(\nu_{s}) 
 &= -\lambda \gamma \left( \KL( \nu_0 \| \hat{\nu}_0 ) + \KL( \hat{\nu}_0 \| \nu_0 ) \right) \\
 &+ \int \lambda\gamma \log \frac{\rd\nu_s}{\rd \nu_0}(\theta) \cdot \frac{\rd \hat{\nu}_0}{\rd \hat{\nu}_s}(\theta)\rd (\hat{\nu}_0 -\nu_0)(\theta).
\end{align*}
The last term is a time-discretization error, which can be small if the step size $\eta$ is small.
Therefore by evaluating this error, we can extend the proof of Theorem \ref{theorem:primal_convergence} and obtain the reduction in objective value per iteration.

In the following, we set $t_0=\lceil\frac{1}{\gamma \eta}\rceil$ and introduce a constant $D_\lambda$ that depends on $\ell_i$ and $\lambda$ (for complete descriptions see Appendix \ref{appendix:discrete-proof}).
In addition, we make the following boundedness assumptions on $h_i$ and $\ell_i$. 
We note the upper bound of 1 can be replaced by any value.
\begin{assumption}\label{assumption:boundedness}
For any $i=1,2,\ldots,n$ and $z \in [1,1]$, 
\[ \| h_i \|_{\infty}\leq 1,~|\ell'_i(z)|\leq 1,~0 \leq \ell''_i(z) \leq 1. \]
\end{assumption}
The following lemma evaluates the decrease in the primal objective value in one step. 
\begin{lemma}\label{lemma:Discrete-Onestep-main}
Let $\{\mu^{(t)}\}_{t=0}^T \subset \cP_2$ be a sequence generated by Algorithm \ref{alg:discrete-time-efp}.
Suppose $\gamma \eta \leq\frac{1}{8}$ holds.
Then, under Assumptions \ref{assumption:regularity} and \ref{assumption:boundedness}, we get for $t\geq t_0$,
\begin{align*}
\cL( \mu^{(t+1)}) - \cL( \mu^{(t)}) \leq 
    -\lambda \gamma\eta \KL( \mu^{(t)} \| \hat{\mu}^{(t)} )
    +\gamma\eta^2 D_\lambda.
\end{align*}
\end{lemma}
Similarly, we can evaluate the dual objective as follows. 
\begin{lemma}\label{lemma:Discrete-Onestep-Dual-main}
Let $\{\mu^{(t)}\}_{t=0}^T \subset \cP_2$ be a sequence generated by Algorithm \ref{alg:discrete-time-efp}.
Suppose $\gamma \eta \leq\frac{1}{8}$ holds.
Then, under Assumptions \ref{assumption:regularity} and \ref{assumption:boundedness}, we get for $t\geq t_0$,
\begin{align*}
-\cD(g_{\mu^{(t+1)}}) +\cD(g_{\mu^{(t)}}) \leq 
    2\gamma \eta^2 D_\lambda.
\end{align*}
\end{lemma}
Combining the above lemmas with Theorem \ref{theorem:duality}, we have
\begin{align*}
   \KL( \mu^{(t+1)} \| \hat{\mu}^{(t+1)} )
   \leq (1-\gamma\eta \KL( \mu^{(t)} \| \hat{\mu}^{(t)})) 
   + \frac{7\gamma \eta^2}{\lambda} D_\lambda.
\end{align*}
Hence we obtain the following primal-dual convergence.
\begin{theorem}\label{theorem:dual-convergence-discrete}
Let $\{\mu^{(t)}\}_{t=0}^T \subset \cP_2$ be a sequence generated by Algorithm \ref{alg:discrete-time-efp}.
Suppose $\gamma \eta \leq\frac{1}{8}$ holds.
Then, under Assumptions \ref{assumption:regularity} and \ref{assumption:boundedness}, we get for $t\geq t_0$,
\begin{align*}
\KL( \mu^{(t)} \| \hat{\mu}^{(t)} )
    \leq (1-\gamma\eta)^{t-t_0}\KL( \mu^{(t_0)} \| \hat{\mu}^{(t_0)} )
    +\frac{7\eta}{\lambda} D_\lambda.
\end{align*}
\end{theorem}

\subsection{Analysis from the Gradient Boosting Viewpoint}
Next we introduce a gradient-boosting viewpoint of the discrete-time EFP (Algorithm \ref{alg:discrete-time-efp}).
Recall the empirical loss is given as $F_0(\mu) = \frac{1}{n}\sum_{i=1}^n \ell_i( \bE_{\mu}[h_i(\theta)] )$.
Observe that the proximal Gibbs distribution $\hat{\mu}$ is characterized by the minimizer of the regularized linear functional:
\begin{equation}\label{eq:characterize-prox-gibbs}
\min_{\mu'}\left\{ \int \frac{\delta F_0}{\delta \mu}(\mu)(\theta) \rd \mu'(\theta) + \lambda \KL(\mu' \| \nu) \right\},
\end{equation}
where $\nu$ is the Gaussian distribution in proportion to $\exp(-\lambda' \|\theta\|_2^2 / \lambda)$.
This can be verified by computing the first variation of the objective and solving the first-order optimality condition: $\frac{\delta F_0}{\delta \mu}(\mu) + \lambda \log (\rd \mu'/\rd \nu)=0$.
Therefore, $\hat{\mu}$ can be interpreted as an approximation of the ``negative gradient'': $-\frac{\delta F_0}{\delta \mu}(\mu)$, and the EFP update \eqref{eq:discrete-time-efp} can be seen as a \textit{regularized gradient boosting} method for the empirical loss $F_0$ over the space of probability distributions. Consequently, we can show the convergence of $F_0(\mu^{(T)})$ based on the Frank-Wolfe argument.

For any $r>0$, we set $\bB_r = \{\xi \in \cP_2 \mid \KL(\xi \| \nu) \leq r \}$ and $\epsilon' = \lambda r + 2 \eta \gamma B^2 L$. Then we have the following convergence result. 
\begin{theorem}\label{theorem:fw-convergence}
Suppose Assumption \ref{assumption:regularity} holds and suppose $\ell_i$ is $L$-Lipschitz smooth, $\ell_i \geq 0$, and $\|h\|_\infty \leq B$. 
Then, for $\{\mu^{(t)}\}_{t=0}^T \subset \cP_2$ generate by Algorithm \ref{alg:discrete-time-efp}, we get
\[ F_0(\mu^{(T)})  \leq \epsilon' + (1-\eta\gamma)^T F_0(\mu^{(0)}) + \inf_{\xi \in \bB_r} F_0(\xi).\]
\end{theorem}
Based on this theorem, we can compute the required time-complexity to achieve an $\epsilon$-error for a given precision $\epsilon>0$: $ F_0(\mu^{(T)}) \leq \inf_{\xi \in \bB_r} F_0(\xi) + \epsilon$.
Concretely, if we set $\lambda = \epsilon/(4r)$, and $\eta \gamma = \epsilon / (8B^2 L)$, then the complexity is 
\[ T = \frac{8B^2 L}{\epsilon} \log \left( \frac{2F_0(\mu_0)}{\epsilon} \right). \]

\remark{\label{rem:fw-convergence}
We briefly outline the following extension of Theorem \ref{theorem:fw-convergence} that applies to Algorithm \ref{alg:implementable-efp}.
Specifically, Theorem \ref{theorem:fw-convergence} can be relaxed so that EFP allows for a tolerance factor in the next iteration $\mu^{(t+1)}$; that is, for a given $\rho \geq 0$, 
\begin{equation}\label{eq:fw-tolerance}
F_0( \mu^{(t+1)})
 \leq \eta \gamma \rho + F_0( (1-\eta \gamma) \mu^{(t)} + \eta \gamma \hat{\mu}^{(t)} ). 
 \end{equation}
For the convenience of explanation in this remark, we denote by $\mu'^{(t+1)}$ the exact update $(1-\eta \gamma) \mu^{(t)} + \eta \gamma \hat{\mu}^{(t)}$ in Algorithm \ref{alg:discrete-time-efp}.
Note that $\mu'^{(t+1)}$ satisfies the above condition with $\rho=0$.
The above relaxation (\ref{eq:fw-tolerance}) yields a guarantee similar to Theorem \ref{theorem:fw-convergence} with $\epsilon'= \rho + \lambda r + 2 \eta \gamma B^2 L$. 
Let us apply this extended version of Theorem \ref{theorem:fw-convergence} to the case of $\mu^{(t+1)} = (1-\eta \gamma) \mu^{(t)} + \eta \gamma \nu^{(t)}$, where $\nu^{(t)}=\frac{1}{m}\sum_{r=1}^m \delta_{\theta_r^{(t)}}$ attained by Langevin Monte Carlo, 
This is to say, we modify Algorithm \ref{alg:discrete-time-efp} by replacing  $\hat{\mu}^{(t)}$ with an empirical distribution $\nu^{(t)}$.
If $\| h_i \|_\infty \leq B$ and $\ell_i$ is $C$-Lipschitz continuous, Pinsker's inequality leads to the following upper-bound on $| F_0(\mu'^{(t+1)}) - F_0(\mu^{(t+1)}) |$:
\begin{align*}
\eta \gamma C B \sqrt{2\KL( \overline{\nu}^{(t)} \| \hat{\mu}^{(t)}) } 
+ \frac{\eta \gamma C}{n}\sum_{i=1}^n | \bE_{\overline{\nu}^{(t)}}[h_i]  - \bE_{\nu^{(t)}}[h_i] |, 
\end{align*}
where we set $\overline{\nu}^{(t)}= \mathrm{Law}(\theta_r^{(t)})$.
Therefore, we can estimate the particle and iteration complexities to satisfy the required precision (\ref{eq:fw-tolerance}) by applying the convergence rate of Langevin Monte Carlo \citep{vempala2019rapid} together with a standard concentration inequality as done in \citet{nitanda2020particle,oko2022psdca}.
}

\section{Experiments}
Note that our theory also guarantees the convergence of the proximal Gibbs distribution $\hat{\underline{\mu}}^{(T-1)}$ to the (regularized) global optimal solution.
Therefore, in the following, we evaluate the performance of $\hat{\underline{\mu}}^{(T-1)}$ as well as $\underline{\mu}^{(T)}$
in the applications of training two-layer neural networks as well as image synthesis using transparent triangles.

\subsection{Two-layer Neural Networks}
As previously remarked, optimizing a two-layer neural network in the mean-field regime is one important application of the EFP algorithm. We write $h(\theta,x)=\sigma( \theta^\top x)$ as one neuron with trainable parameters $\theta\in\bR^d$. 
In this setting, the mean-field model $\bE_{\mu}[h(\theta,\cdot)]$ is an infinite-width neural network, and $\frac{1}{m}\sum_{r=1}^m h(\theta_r,\cdot)$ is an $M$-particle approximation of the infinite-width limit (i.e., a neural network with $m$ hidden neurons). 

We consider a student-teacher setting where the labels are generated by a teacher network with cosine activation; we set $n=500, d=5$. The student model is a two-layer network with tanh activation and width $m=1000$. The training objective is to minimize the empirical squared error, and the regularization hyperparameters are set to $\lambda=\lambda'=0.01$. 
We optimize the neural network using EFP with an outer-loop step size $\eta\gamma=0.01$. At each iteration, we approximate the proximal Gibbs measure $\hat{\mu}_t$ via the Langevin Monte Carlo algorithm with step size $\eta'=0.01$. 

\begin{figure}[!htb]
\centering
\vspace{-1mm}
\includegraphics[width=0.8\linewidth]{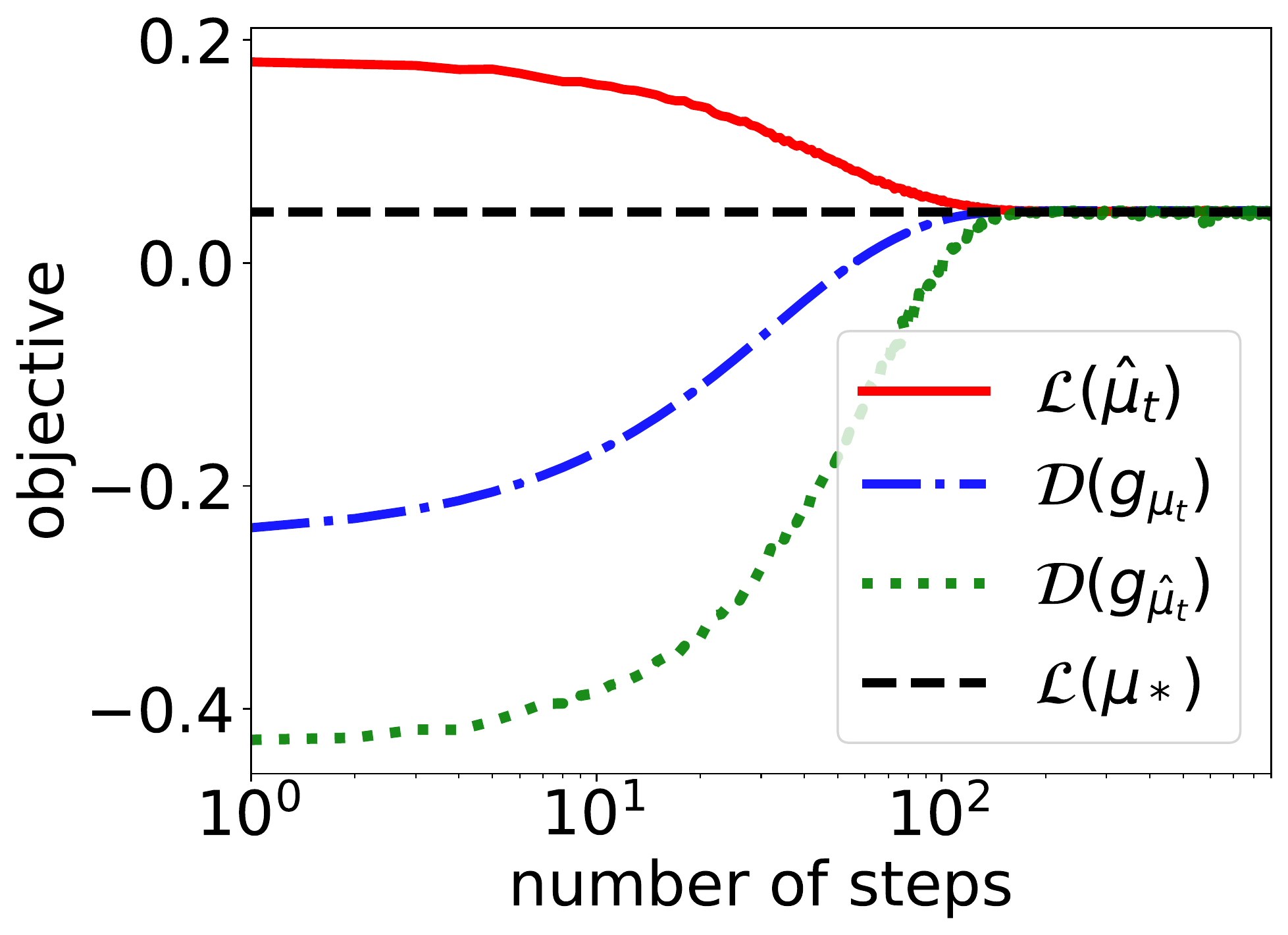}
\vspace{-2mm}
\caption{Illustration of primal-dual convergence: learning two-layer neural net with EFP ($\lambda=\lambda'=10^{-2}$).}
\label{fig:duality-gap}
\end{figure} 

In Figure~\ref{fig:duality-gap} we report the primal and dual objective values, in which the entropy term is computed by the $k$-nearest neighbors estimator \cite{kozachenko1987sample}. The optimal value $\cL(\mu_*)$ is approximated using the particle dual averaging (PDA) algorithm \cite{nitanda2020particle} which also globally optimizes the objective \eqref{prob:erm-primal}. 
Observe that the duality gap vanishes as predicted by our theoretical analysis.

\subsection{Image Synthesis with Triangles}
We also consider the image synthesis experiment in \citet{tian2022modern}, where the goal is to ``approximate'' an image via integrating transparent triangles. To draw the triangles, we utilize a differentiable render \citep{laine2020modular} that provides a differentiable map to translate the parameters $\theta$ (representing the colors and vertexes) into a transparent triangle $h(\theta)$.
Then, we can construct an image $\bE_{\theta \sim \mu}[h(\theta)]$ by integrating $h(\theta)$ with the probability distribution $\mu$.
For simple explanation, we take a $W\times H$ image of a single channel, i.e., $h(\theta)=\{h_{ij}(\theta)\}_{i,j} \in \bR^{WH}$.
We formalize the problem of approximating a given image $J \in \bR^{WH}$ as a regularized regression problem using the squared error: 
\begin{align*} 
\frac{1}{WH}\sum_{i,j}\ell_{i,j}(\bE_{\mu}[h_{ij}(\theta)]) 
= \frac{1}{WH}\sum_{i,j}(J_{i,j} - \bE_{\mu}[h_{ij}(\theta)])^2.
\end{align*}
Then, we can apply Algorithm \ref{alg:implementable-efp} to minimize this function with regularization. 
We adopt the image ``Mona Lisa'' of size $256\times256$ as the target. See Figure \ref{fig:mona_lisa} for the original target image and one generated image $H^{(T)} = \{H_{i,j}^{(t)}\}_{i,j}$ using EFP (Algorithm \ref{alg:implementable-efp}).
\begin{figure}[ht]
\center
\begin{tabular}{cc}
\begin{minipage}[t]{0.45\linewidth}
\centering
\includegraphics[width=\textwidth]{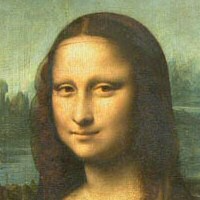} \\
\end{minipage} 
\hspace{1mm}
\begin{minipage}[t]{0.45\linewidth}
\centering
\includegraphics[width=\textwidth]{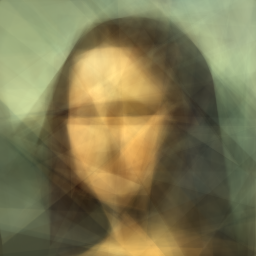}
\end{minipage} 
\end{tabular}
\caption{We run Algorithm \ref{alg:implementable-efp} with $\lambda=10^{-5},~\lambda'=10^{-4},~T=2000,~S=10,~m=1000,~\eta\cdot\gamma=0.01$ to fit the target image. As for the step size for Langevin Monte Carlo, we used cosine annealing from $0.1$ to $0.01$. The figure depicts the target image ``Mona Lisa'' (left) and an obtained image (right).}
\label{fig:mona_lisa}
\end{figure}

In Figure \ref{fig:mona_lisa_comparison} we compare the generated images $H^{(T)}$ with $\frac{1}{m}\sum_{r=1}^m h(\theta_r^{(T-1)})$ using different number of triangles of $m=200$ and $m=1000$; we also report the generated image under both $\underline{\mu}^{(T)}$ and the Gibbs distribution, $\hat{\underline{\mu}}^{(T-1)}$ as explained in Section \ref{subsec:algorithmic-advantage}.
Figure \ref{fig:mona_lisa_comparison} depicts these generated images.
Finally, we plot the squared error of each image during optimization in Figure 
 \ref{fig:loss_comparison}.
\begin{figure}[ht]
\center
\begin{tabular}{cccc}
\begin{minipage}[t]{0.23\linewidth}
\centering
\includegraphics[width=\textwidth]{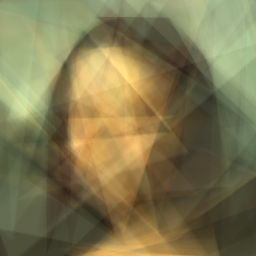} \\
{\scriptsize $\underline{\mu}^{(T)}$}
\end{minipage} 
\begin{minipage}[t]{0.23\linewidth}
\centering
\includegraphics[width=\textwidth]{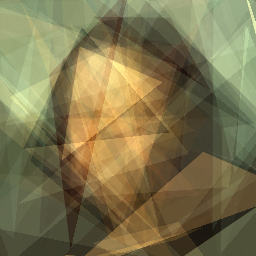}
{\scriptsize $\hat{\underline{\mu}}^{(T-1)}$}
\end{minipage} 
\hspace{1.5mm}
\begin{minipage}[t]{0.23\linewidth}
\centering
\includegraphics[width=\textwidth]{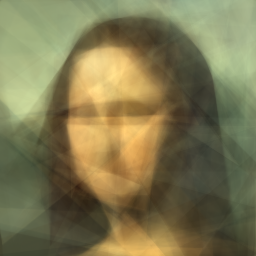} \\
{\scriptsize $\underline{\mu}^{(T)}$}
\end{minipage} 
\begin{minipage}[t]{0.23\linewidth}
\centering
\includegraphics[width=\textwidth]{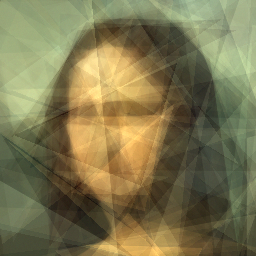}
{\scriptsize $\hat{\underline{\mu}}^{(T-1)}$}
\end{minipage} 
\end{tabular}
\vspace{-2mm}
\caption{Comparison of images generated by $\underline{\mu}^{(T)}$ and $\hat{\underline{\mu}}^{(T-1)}$. We set $m=200$ for the first two figures on the left, and $m=1000$ for the two on the right. In both cases, the left and right figures correspond to images by $\underline{\mu}^{(T)}$ and $\hat{\underline{\mu}}^{(T-1)}$, respectively.}
\label{fig:mona_lisa_comparison}
\end{figure}

\begin{figure}[htb]
\center
\begin{tabular}{cc}
\begin{minipage}[t]{0.47\linewidth}
\centering
\includegraphics[width=\textwidth]{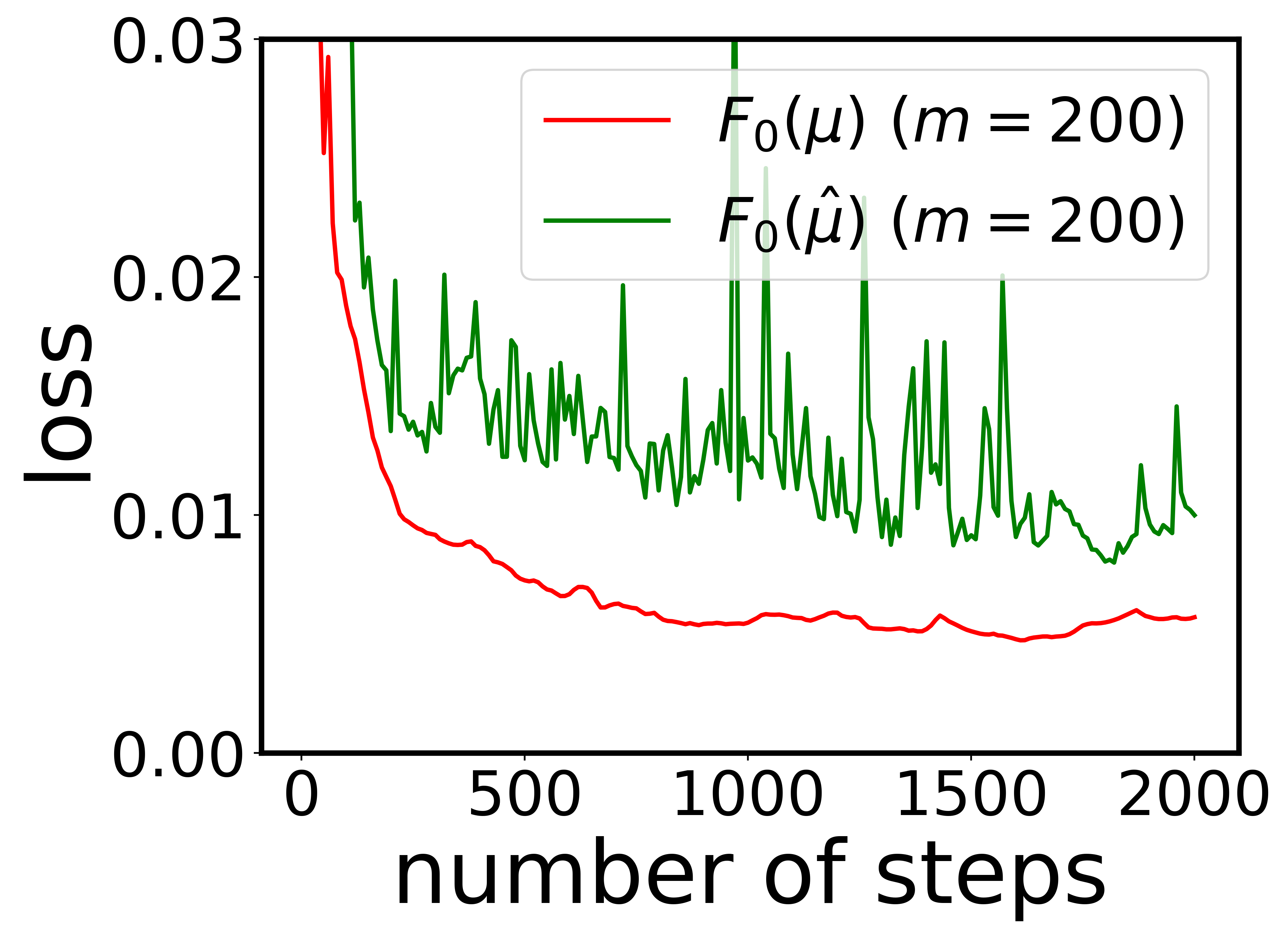} \\
\end{minipage} 
\begin{minipage}[t]{0.47\linewidth}
\centering
\includegraphics[width=\textwidth]{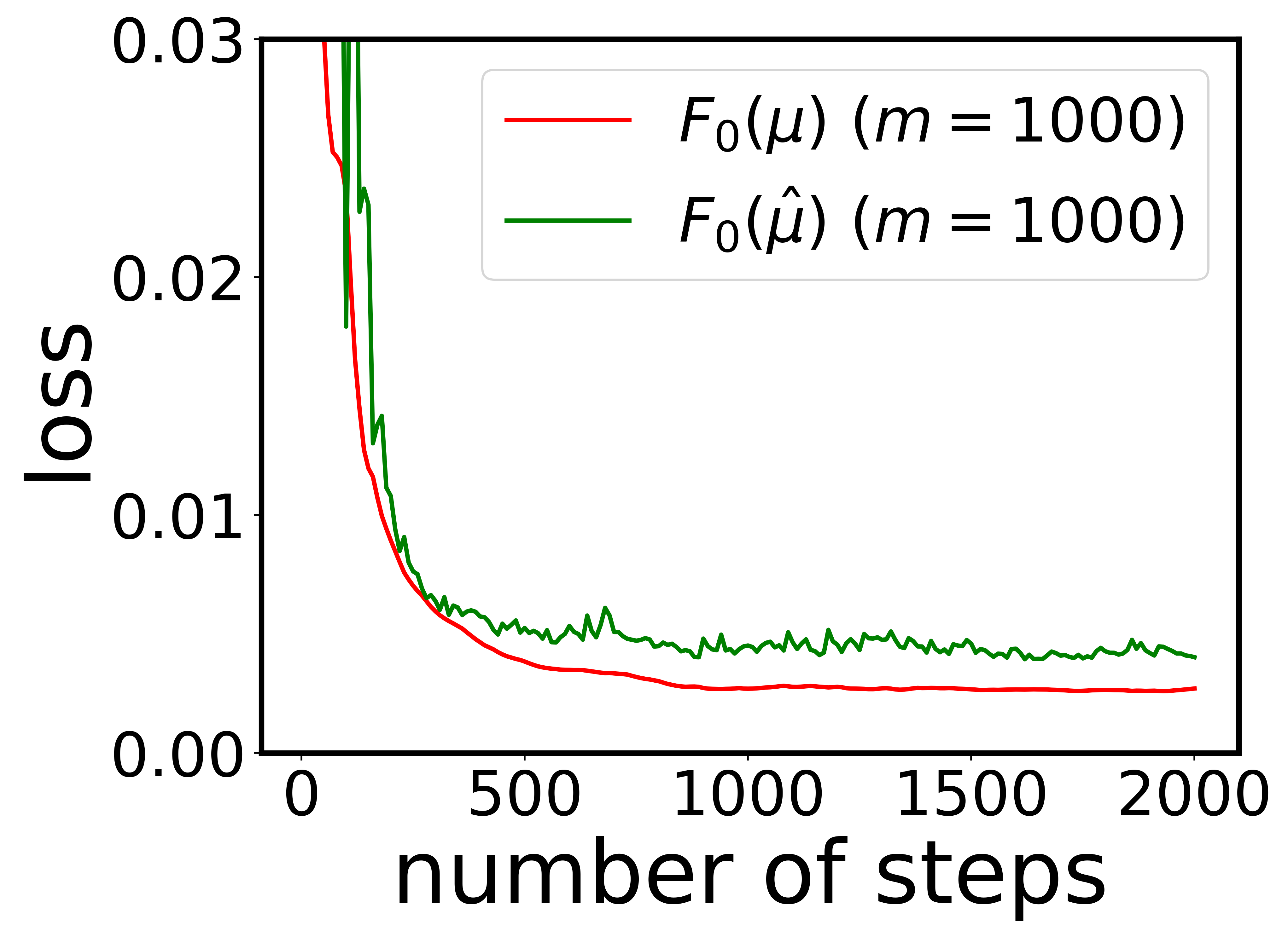} 
\end{minipage} 
\end{tabular}
\vspace{-2mm}
\caption{Squared error achieved by $\underline{\mu}^{(T)}$ and $\hat{\underline{\mu}}^{(T-1)}$ with the number of triangles (particle size) $m=200,1000$.}
\label{fig:loss_comparison}
\end{figure} 

We observe that EFP is capable of generating reasonable images. Moreover, as the number of particle $m$ increases, the quality of images generated by both $\underline{\mu}^{(T)}$ and $\hat{\underline{\mu}}^{(T-1)}$ improve and the discrepancy between them becomes smaller. 
This is consistent with our theory, which guarantees the convergence of both distributions to the optimal solution when  $m$ becomes large.

\section*{Conclusion}
In this paper, we presented a concise primal-dual analysis of the entropic fictitious play (EFP) algorithm for the finite-sum minimization problem. Specifically, we showed the convergence in both continuous- and discrete-time in terms of the primal and dual objectives, which is equivalent to the convergence of the KL divergence between $\mu$ and $\hat{\mu}$. Our analysis also leads to an efficient implementation of EFP which resolves the huge memory consumption issue. Moreover, we introduced a gradient-boosting perspective of EFP which gives a new discrete-time global convergence analysis; under this perspective, EFP can be regarded as the memory-efficient gradient-boosting method performing on the space of probability distributions. 
Finally, we empirically verified our theory in the application of training two-layer neural networks and an image synthesis using transparent triangles.
An interesting future work is the generalization analysis for mean-field models. Especially, learnability in the classification setting is of interest for the combination of our convergence result and margin theory.


\bigskip

\section*{Acknowledgment}
AN was partially supported by JSPS Kakenhi (22H03650) and JST-PRESTO (JPMJPR1928). 
KO was partially supported by IIW program of the Univ. of Tokyo.
DW was partially supported by a Borealis AI Fellowship. 
TS was partially supported by JSPS KAKENHI (20H00576) and JST CREST.

\bibliography{ref}
\bibliographystyle{apalike}

\clearpage
\onecolumn
\renewcommand{\thesection}{\Alph{section}}
\renewcommand{\thesubsection}{\Alph{section}. \arabic{subsection}}
\renewcommand{\thetheorem}{\Alph{theorem}}
\renewcommand{\thelemma}{\Alph{lemma}}
\renewcommand{\theproposition}{\Alph{proposition}}
\renewcommand{\thedefinition}{\Alph{definition}}
\renewcommand{\thecorollary}{\Alph{corollary}}
\renewcommand{\theassumption}{\Alph{assumption}}
\renewcommand{\theexample}{\Alph{example}}

\setcounter{section}{0}
\setcounter{subsection}{0}
\setcounter{theorem}{0}
\setcounter{lemma}{0}
\setcounter{proposition}{0}
\setcounter{definition}{0}
\setcounter{corollary}{0}
\setcounter{assumption}{0}

{
\newgeometry{top=1in, bottom=1in,left=1.in,right=1.in}   

\newpage

\allowdisplaybreaks

\linewidth\hsize
{\centering \Large\bfseries Appendix: Primal and Dual Analysis of Entropic Fictitious Play for Finite-sum Problems \par}
\section{Analysis in discrete-time}\label{appendix:discrete-proof}
\begin{lemma}\label{Lemma:Discrete-BoundedRatio}
    Suppose Assumption \ref{assumption:boundedness} holds.
    Then, any two probability distributions $\mu_1,\mu_2$ satisfy
    \begin{align}\label{eq:Discrete-BoundedRatio-1}
      C_\lambda^{-1}  \leq  \frac{\hat{\mu}_1(\theta)}{\hat{\mu}_2(\theta)
      } \leq C_\lambda,
    \end{align}
    with $C_\lambda:=\exp(4/\lambda)$.

    Moreover, let $\mu_1, \mu_2$ be two probability distributions, and $\mu:=(1-t)\mu_1+t\mu_2$ be an interpolation with $t\in [0,\frac12]$. Then, we have that
    \begin{align}\label{eq:Discrete-BoundedRatio-2}
      - 4t \leq  \frac{\hat{\mu}(\theta)}{\hat{\mu}_1(\theta)} -1 \leq 4t C_\lambda.
    \end{align} 
\end{lemma}

\begin{proof}
    Because of $\|h_i\|_\infty \leq 1$ and $|(g_\mu)_i|=|\ell'_i(x)| \leq 1$, the conclusion follows immediately by the following formulation:
   \begin{align*}
        \frac{\hat{\mu}_1(\theta)}{\hat{\mu}_2(\theta)}= \frac{\exp\left( - \frac{1}{\lambda} \left( \frac{1}{n}\sum_{i=1}^n h_{i}(\theta)g_{\mu_1,i} 
+ \lambda'\|\theta\|_2^2 \right) \right)}{\exp\left( - \frac{1}{\lambda} \left( \frac{1}{n}\sum_{i=1}^n h_{i}(\theta)g_{\mu_2,i}
+ \lambda'\|\theta\|_2^2 \right) \right)}
\frac{\int\exp\left( - \frac{1}{\lambda} \left( \frac{1}{n}\sum_{i=1}^n h_{i}(\theta')g_{\mu_2,i}
+ \lambda'\|\theta'\|_2^2 \right)\mathrm{d}\theta'\right)}{\int\exp\left( - \frac{1}{\lambda} \left( \frac{1}{n}\sum_{i=1}^n h_{i}(\theta'')g_{\mu_1,i}
+ \lambda'\|\theta''\|_2^2 \right)\mathrm{d}\theta''\right)}.
    \end{align*}

    In the same way, to bound $\frac{\hat{\mu}(\theta)}{\hat{\mu_1}(\theta)}$, we consider the density ratio:
    \begin{align}\label{eq:Discrete-BoundedRatio-3}
        \frac{\hat{\mu}(\theta)}{\hat{\mu}_1(\theta)}= \frac{\exp\left( - \frac{1}{\lambda} \left( \frac{1}{n}\sum_{i=1}^n h_{i}(\theta)g_{\mu,i}
+ \lambda'\|\theta\|_2^2 \right) \right)}{\exp\left( - \frac{1}{\lambda} \left( \frac{1}{n}\sum_{i=1}^n h_{i}(\theta)g_{\mu_1,i} 
+ \lambda'\|\theta\|_2^2 \right) \right)}
\frac{\int\exp\left( - \frac{1}{\lambda} \left( \frac{1}{n}\sum_{i=1}^n h_{i}(\theta')g_{\mu_1,i}
+ \lambda'\|\theta'\|_2^2 \right)\mathrm{d}\theta'\right)}{\int\exp\left( - \frac{1}{\lambda} \left( \frac{1}{n}\sum_{i=1}^n h_{i}(\theta'')g_{\mu,i}
+ \lambda'\|\theta''\|_2^2 \right)\mathrm{d}\theta''\right)}.
    \end{align}
    Note that 
      \begin{align*}
          |g_{\mu,i} - g_{\mu_1,i}|&=\left|\ell'_i\left(\int h_i(\theta)\mathrm{d}\mu(\theta)\right) - \ell'_i\left(\int h_i(\theta)\mathrm{d}\mu_1(\theta)\right)\right|\\ & \leq \left|\int h_i(\theta)\mathrm{d}\mu(\theta) - \int h_i(\theta)\mathrm{d}\mu_1(\theta)\right|
          \\ & =
          t\left|\int h_i(\theta)\mathrm{d}\mu_2(\theta) - \int h_i(\theta)\mathrm{d}\mu_1(\theta)\right|
          \\ & \leq 2t.
      \end{align*}  
      Therefore, the equation (\ref{eq:Discrete-BoundedRatio-3}) is upper and lower bounded by $\exp(4t/\lambda)$ and $\exp(-4t/\lambda)$, respectively.
      Note that we have $\exp(4t/\lambda) \leq 1+4t\exp(4/\lambda)$ and $\exp(-4t/\lambda) \geq 1-4t$ if $0\leq t\leq \frac14$ holds, which concludes the proof.
\end{proof}
\begin{lemma}\label{Lemma:Discrete-FirstPhase}
    Suppose Assumption \ref{assumption:boundedness} and $\gamma \eta <1$ hold.
    Then, for $t\geq t_0:=\lceil\frac{1}{\gamma \eta}\rceil$, we get $\frac{1}{2C_\lambda}\hat{\mu}^{(t)} \leq \mu^{(t)}$. 
\end{lemma}
\begin{proof}
    First, we observe that
    \begin{align}\nonumber
        \mu^{(t)} & = (1-\gamma \eta) \mu^{(t-1)} + \eta \gamma \hat{\mu}^{(t-1)}
        \\ & = \cdots =  (1-\gamma \eta)^t \mu^{(0)} + \eta \gamma \sum_{s=0}^{t-1} (1-\gamma \eta)^{t-s-1}\hat{\mu}^{(s)}.
        \label{eq:Discrete-FirstPhase-1}
    \end{align}
    According to the bounds \eqref{eq:Discrete-BoundedRatio-1} of \cref{Lemma:Discrete-BoundedRatio}, $\hat{\mu}^{s}(\theta) \geq C_\lambda^{-1}\hat{\mu}^{(t)}$ holds for each $s$ and $\theta$.
    Therefore, \eqref{eq:Discrete-FirstPhase-1} yields that
    \begin{align*}
        \mu^{(t)} \geq \eta \gamma \sum_{s=0}^{t-1} (1-\gamma \eta)^{t-s-1}C_\lambda^{-1}\hat{\mu}^{(t)}
        \geq (1-(1-\gamma \eta)^{t})C_\lambda^{-1}\hat{\mu}^{(t)}.
    \end{align*}
    By letting $t\geq t_0:=\lceil\frac{1}{\gamma \eta}\rceil \geq \log_{1-\gamma \eta} \frac{1}{2}$, we obtain that $\mu^{(t)} \geq \frac{1}{2C_\lambda}\hat{\mu}^{(t)}$.
    
\end{proof}
We then present the analysis for each iteration of \cref{alg:discrete-time-efp} for the primal and dual objectives.
\begin{lemma}\label{lemma:Discrete-Onestep}
Suppose Assumption \ref{assumption:boundedness} and $\gamma \eta \leq \frac{1}{8}$ hold.
Then, for $t\geq t_0:=\lceil\frac{1}{\gamma \eta}\rceil$, we get
\begin{align*}
\cL( \mu^{(t+1)}) - \cL( \mu^{(t)}) \leq 
    -\lambda \gamma\eta \KL( \mu^{(t)} \| \hat{\mu}^{(t)} )
    +\gamma\eta^2(3C_\lambda+5)
    .
\end{align*}
\end{lemma}
\begin{proof}
We consider the following continuous interpolation between $\mu^{(t)}$ and $\mu^{(t+1)}$. Let $\nu_0 = \mu^{(t)}$, and $\nu_\eta = \nu_0 + \gamma \eta(\hat{\nu}_0 -\nu_0) = \mu^{(t)} + \eta\gamma (\hat{\mu}^{(t)} -\mu^{(t)})$. 
We then revisit the analysis on primal convergence (\cref{theorem:primal_convergence}).
\begin{align}
    \frac{\rd }{\rd \eta}\cL(\nu_\eta) 
    \nonumber &= \int \frac{\delta \cL}{\delta \mu}(\nu_\eta)(\theta)\frac{\rd \nu_\eta}{\rd \eta}(\theta) \\
    \nonumber &= \int \lambda\gamma \log \frac{\rd\nu_\eta}{\rd \hat{\nu}_\eta}(\theta) \rd (\hat{\nu}_0 -\nu_0)(\theta) \\    
   &= \int \lambda\gamma \log \frac{\rd\nu_0}{\rd \hat{\nu}_0}(\theta) \rd (\hat{\nu}_0 -\nu_0)(\theta)
     + \int \lambda\gamma \log \frac{\rd\nu_\eta}{\rd \nu_0}(\theta) \rd (\hat{\nu}_0 -\nu_0)(\theta)
   + \int \lambda\gamma \log \frac{\rd \hat{\nu}_0}{\rd \hat{\nu}_\eta}(\theta) \rd (\hat{\nu}_0 -\nu_0)(\theta)
   \label{eq:Discrete-Onestep-1}
\end{align}
The first term is equal to $-\lambda \gamma \left( \KL( \nu_0 \| \hat{\nu}_0 ) + \KL( \hat{\nu}_0 \| \nu_0 ) \right)$ as previously.

To bound the second term, we bound the density ratio $\frac{\rd\nu_\eta}{\rd \nu_0}(\theta)$ as
\begin{align*}
     \frac{\rd\nu_\eta}{\rd \nu_0}(\theta) = \frac{\rd (\nu_0 - \gamma \eta (\nu_0-\hat{\nu}_0)) }{\rd \nu_0}(\theta)
     = \frac{(1-\gamma \eta )\rd \mu^{(t)} }{\rd \mu^{(t)}}(\theta) + \gamma \eta\frac{\rd \hat{\mu}^{(t)} }{\rd \mu^{(t)}}(\theta)
     \begin{cases}
     \geq 1-\gamma \eta ,\\
     \leq  (1-\gamma \eta) + 2C_\lambda \gamma \eta
     = 1 + \gamma\eta (2C_\lambda-1),
     \end{cases}
\end{align*}
where we used \cref{Lemma:Discrete-FirstPhase} the upper bound.
Therefore, we have that
\begin{align*}
  \int \lambda\gamma \log \frac{\rd\nu_\eta}{\rd \nu_0}(\theta) \rd (\hat{\nu}_0 -\nu_0)(\theta)
  \leq \log (1 + \gamma\eta (2C_\lambda-1)) -\log (1-\gamma \eta)\leq \gamma\eta(2C_\lambda+2) ,
\end{align*}
where we used that $\log (1+x) \leq x$ and $\log (1-x) \geq 2x\ (0 \leq x\leq \frac12)$ for the last inequality.

The third term is bounded in a similar manner.
The bounds \eqref{eq:Discrete-BoundedRatio-2} from \cref{Lemma:Discrete-BoundedRatio} yield that
\begin{align*}
     \int \lambda\gamma \log \frac{\rd \hat{\nu}_0}{\rd \hat{\nu}_\eta}(\theta) \rd (\hat{\nu}_0 -\nu_0)(\theta)
    \leq \log (1+4\gamma\eta C_\lambda)
     -\log (1-4\gamma\eta ).
\end{align*}
RHS is further bounded by $4\gamma\eta (C_\lambda+2)$, where we used the fact that $\log (1+x) \leq x$ and $\log (1-x) \geq 2x\ (0 \leq x\leq \frac12)$.

Putting it all together, \eqref{eq:Discrete-Onestep-1} is bounded by 
\begin{align*}
   \frac{\rd }{\rd \eta}\cL(\nu_\eta) \leq  -\lambda \gamma \left( \KL( \nu_0 \| \hat{\nu}_0 ) + \KL( \hat{\nu}_0 \| \nu_0 ) \right)
    +\gamma\eta(6C_\lambda+10)
    \leq 
    -\lambda \gamma\KL( \mu^{(t)} \| \hat{\mu}^{(t)} )+\gamma\eta(6C_\lambda+10).
\end{align*}
By integrating LHS and RHS, we obtain that
\begin{align*}
   \cL(\nu_\eta)-\cL(\nu_0)
 \leq 
     -\lambda \gamma\eta \KL( \mu^{(t)} \| \hat{\mu}^{(t)} ) +\gamma\eta^2(3C_\lambda+5),
\end{align*}
and $\cL( \mu^{(t+1)}) - \cL( \mu^{(t)})=\cL(\nu_\eta)-\cL(\nu_0)$ concludes the proof.
\end{proof}
\begin{lemma}\label{lemma:Discrete-Onestep-Dual}
Suppose Assumption \ref{assumption:boundedness} and $\gamma \eta \leq \frac{1}{8}$ hold.
Then, for $t\geq t_0:=\lceil\frac{1}{\gamma \eta}\rceil$, we get
\begin{align*}
-\cD(g_{\mu^{(t+1)}}) +\cD(g_{\mu^{(t)}}) \leq 
    2\gamma \eta^2 (1+2C_\lambda)
    .
\end{align*}
\end{lemma}
\begin{proof}
    Let $\nu_\eta$ be defined in the same way as previously.
    We consider the evolution of $\frac{\rd }{\rd \eta} \cD(g_{\nu_\eta}) $ based on the analysis for Theorem \ref{theorem:dual-convergence}.
    We have that
\begin{align}\nonumber
    -\frac{\rd }{\rd \eta} \cD(g_{\nu_\eta}) 
   & = -\nabla \cD(g_{\nu_\eta})^\top \frac{\rd g_{\nu_\eta}}{\rd \nu} \\
\nonumber    &= \frac{1}{n}\sum_{i=1}^n \left( \ell_i^{*\prime}(g_{\nu_\eta,i}) - \frac{1}{\int q_{g_{\nu_\eta}}(\theta)\rd\theta} \int h_i(\theta) q_{g_{\nu_\eta}}(\theta)\rd\theta \right) \cdot \frac{\rd}{\rd \eta} \ell_i'( \bE_{\nu_\eta}[h_i(\theta)]) \\
 \label{eq:Discrete-Onestep-Dual-1}   &= -\frac{\gamma}{n}\sum_{i=1}^n (\bE_{\nu_\eta}[h_i(\theta)] - \bE_{\hat{\nu}_\eta}[h_i(\theta)])(\bE_{\nu_0}[h_i(\theta)] - \bE_{\hat{\nu}_0}[h_i(\theta)]) \ell_i''( \bE_{\nu_\eta}[h_i])
    .
\end{align}
where we used $\hat{\nu}_\eta \propto q_{g_{\nu_\eta}}(\theta)\rd \theta$, $\ell_i^{*\prime} = (\ell_i')^{-1}$, and
\begin{align*}
\frac{\rd}{\rd \eta} \ell_i'( \bE_{\nu_\eta}[h_i(\theta)])
= \ell_i''( \bE_{\nu_\eta}[h_i]) \int h_i(\theta) \frac{\rd \nu_\eta}{\rd \eta}(\theta)
= \ell_i''( \bE_{\nu_\eta}[h_i]) (\bE_{\hat{\nu}_0}[h_i] - \bE_{\nu_0}[h_i] ).
\end{align*}
We bound the difference between $(\bE_{\nu_\eta}[h_i(\theta)] - \bE_{\hat{\nu}_\eta}[h_i(\theta)])$ and $(\bE_{\nu_0}[h_i(\theta)] - \bE_{\hat{\nu}_0}[h_i(\theta)])$.
First, $\bE_{\nu_\eta}[h_i(\theta)]$ is evaluated as
\begin{align}
    \nonumber 
    &\bE_{\nu_\eta}[h_i(\theta)] = (1-\gamma\eta)\bE_{\nu_0}[h_i(\theta)] + \gamma\eta\bE_{\hat{\nu}_0}[h_i(\theta)]
    \\
    &\therefore 
    |\bE_{\nu_\eta}[h_i(\theta)]-\bE_{\nu_0}[h_i(\theta)]| \leq 2\gamma \eta,
    \label{eq:Discrete-Onestep-Dual-2}
\end{align}
where we used $\|h_i\|_\infty \leq 1$.
Next, according to \eqref{eq:Discrete-BoundedRatio-2} of \cref{Lemma:Discrete-BoundedRatio}, we have that
\begin{align}
    &|\bE_{\hat{\nu}_\eta}[h_i(\theta)]- \bE_{\hat{\nu}_0}[h_i(\theta)]|
    =\left|\int h_i \mathrm{d}(\hat{\nu}_\eta - \hat{\nu}_0)\right|
    \leq 
    4\gamma \eta C_\lambda\left|\int h_i \mathrm{d}\hat{\nu}_0\right|
    \leq 4\gamma \eta C_\lambda.
    \label{eq:Discrete-Onestep-Dual-3}
\end{align}
Now applying \eqref{eq:Discrete-Onestep-Dual-2} and \eqref{eq:Discrete-Onestep-Dual-3} to \eqref{eq:Discrete-Onestep-Dual-1}, we get
\begin{align*}
-\frac{\rd }{\rd \eta} \cD(g_{\nu_\eta})
=-\frac{\gamma}{n}\sum_{i=1}^n (\bE_{\nu_0}[h_i(\theta)] - \bE_{\hat{\nu}_0}[h_i(\theta)]) \ell_i''( \bE_{\nu_\eta}[h_i])^2
+4\gamma \eta(1+2C_\lambda)L\leq 4\gamma \eta(1+2C_\lambda).
\end{align*}
where we also used $\|h_i\|_\infty \leq 1$ and $0\leq\ell_i''\leq 1$.
By integrating LHS and RHS, we obtain that
\begin{align*}
   -\cD(g_{\nu_\eta}) +\cD(g_{\nu_0})
 \leq 2\gamma \eta^2 (1+2C_\lambda).
\end{align*}
$-\cD(g_{\mu^{(t+1)}}) +\cD(g_{\mu^{(t)}})=-\cD(g_{\nu_\eta}) +\cD(g_{\nu_0})$ by definition yields the assertion.
\end{proof}

\begin{theorem}
Suppose Assumption \ref{assumption:boundedness} and $\gamma \eta \leq \frac{1}{8}$ hold.
Then, for $t\geq t_0:=\lceil\frac{1}{\gamma \eta}\rceil$, we get
    \begin{align*}
    \KL( \mu^{(t)} \| \hat{\mu}^{(t)} )
        \leq (1-\gamma\eta)^{t-t_0}\KL( \mu^{(t_0)} \| \hat{\mu}^{(t_0)} )
        +\frac{7\eta}{\lambda} (1+C_\lambda)
    \end{align*}
    for all $t\geq t_0$, where $t_0 =\lceil\frac{1}{\gamma \eta}\rceil$ and $C_\lambda = \exp(4/\lambda)$.
\end{theorem}
\begin{proof}
    According to \cref{lemma:Discrete-Onestep,lemma:Discrete-Onestep-Dual}, we have that
    \begin{align*}
        \cL( \mu^{(t+1)})-\cD(g_{\mu^{(t+1)}})
        - (\cL( \mu^{(t)})-\cD(g_{\mu^{(t+1)}}))
        \leq -\lambda \gamma\eta \KL( \mu^{(t)} \| \hat{\mu}^{(t)} )+7\gamma \eta^2 (1+C_\lambda)
        .
    \end{align*}
    LHS is equal to $\lambda \KL( \mu^{(t+1)} \| \hat{\mu}^{(t+1)} )-\lambda\KL( \mu^{(t)} \| \hat{\mu}^{(t)} )$, by using \cref{theorem:duality}.
    Thus, we obtain that
    \begin{align}\label{eq:Discrete-PrimalDual-1}
       \KL( \mu^{(t+1)} \| \hat{\mu}^{(t+1)} )-\KL( \mu^{(t)} \| \hat{\mu}^{(t)} )
        \leq -\gamma\eta \KL( \mu^{(t)} \| \hat{\mu}^{(t)} )+ \frac{7\gamma \eta^2}{\lambda} (1+C_\lambda),
    \end{align}
    and Gr\"{o}nwall’s inequality yields that
    \begin{align}
    \left(\KL( \mu^{(t)} \| \hat{\mu}^{(t)} )-\frac{\eta}{\lambda} (7+7C_\lambda)\right)
        \leq (1-\gamma\eta)^{t-t_0} \left(\KL( \mu^{(t_0)} \| \hat{\mu}^{(t_0)} )-\frac{7\eta}{\lambda} (1+C_\lambda)\right).
    \end{align}
   Therefore, we obtain that
    \begin{align*}
    \KL( \mu^{(t)} \| \hat{\mu}^{(t)} )
        \leq (1-\gamma\eta)^{t-t_0} \KL( \mu^{(t_0)} \| \hat{\mu}^{(t_0)} )
        +\frac{7\eta}{\lambda} (1+C_\lambda)
    \end{align*}
    for all $t\geq t_0$.
\end{proof}
\begin{corollary}
    In order to make the dualty gap smaller than $\varepsilon$, i.e. $\cL( \mu^{(t)})-\cD(g_{\mu^{(t)}})=\lambda \KL( \mu^{(t)} \| \hat{\mu}^{(t)} )\leq \varepsilon$, it suffices to take $\eta = \frac{\varepsilon}{14(1+C_\lambda)}$ and $t\geq 14\gamma (1+C_\lambda)\varepsilon^{-1}\log(2e\lambda\varepsilon^{-1}\KL( \mu^{(t_0)} \| \hat{\mu}^{(t_0)} ))$.
\end{corollary}

\section{Gradient Boosting Viewpoint}\label{appendix:gb_viewpoint}
We here relax the next iteration $\mu^{(t+1)}$ to an inexact version. 
That is, we replace $\mu^{(t+1)}=(1-\eta \gamma) \mu^{(t)} + \eta \gamma \hat{\mu}^{(t)}$ in Algorithm \ref{alg:discrete-time-efp} with a probability distribution that satisfies
\begin{equation}\label{eq:fw-tolerance-appendix}
F_0( \mu^{(t+1)})
 \leq \eta \gamma \rho + F_0( (1-\eta \gamma) \mu^{(t)} + \eta \gamma \hat{\mu}^{(t)} ). 
\end{equation}
 
Then, we prove the generalized version of Theorem \ref{theorem:fw-convergence} which can apply to Algorithm \ref{alg:implementable-efp} with the above modification. 
\begin{theorem}\label{theorem:fw-convergence-appendix}
Suppose Assumption \ref{assumption:regularity} holds and suppose $\ell_i$ is $L$-Lipschitz smooth, $\ell_i \geq 0$, and $\|h\|_\infty \leq B$. 
Then, for $\{\mu^{(t)}\}_{t=0}^T \subset \cP_2$ defined above, we get
\[ F_0(\mu^{(T)})  \leq \epsilon' + (1-\eta\gamma)^T F_0(\mu^{(0)}) + \inf_{\xi \in \bB_r} F_0(\xi), \]
where $\epsilon'= \rho + \lambda r + 2 \eta \gamma B^2 L$.
\end{theorem}
\begin{proof}
For notational simplicity, we denote $h_{i,\mu} = \bE_\mu[h_i(\theta)]$.
By the $L$-Lipschitz smoothness of $\ell_i$, boundedness $|h_i(\theta)| \leq B$, and the definition of $F_0$, we get for $\mu, \mu' \in \cP_2$,
\begin{align*}
 F_0(\mu') 
 &\leq F_0(\mu) + \frac{1}{n}\sum_{i=1}^n \ell_i'( h_{i,\mu} )( h_{i,\mu'} - h_{i,\mu}) + \frac{L}{2n} \sum_{i=1}^n (h_{i,\mu'} - h_{i,\mu})^2 \\
 &= F_0(\mu) + \frac{1}{n}\sum_{i=1}^n \ell_i'( h_{i,\mu} )\int h_i(\theta )\rd (\mu' - \mu)(\theta) 
 + \frac{L}{2n} \sum_{i=1}^n \left(\int h_i(\theta )\rd (\mu' - \mu)(\theta) \right)^2 \\
 &\leq F_0(\mu) + \int \frac{\delta F_0}{\delta \mu}(\mu)(\theta) \rd (\mu' - \mu)(\theta) 
 + \frac{L}{2n} \sum_{i=1}^n \left(\int | h_i(\theta )| \left| \frac{\rd (\mu'-\mu)}{\rd\theta}(\theta) \right| \rd\theta \right)^2 \\
 &\leq F_0(\mu) + \int \frac{\delta F_0}{\delta \mu}(\mu)(\theta) \rd (\mu' - \mu)(\theta) 
 + \frac{B^2 L}{2n} \sum_{i=1}^n \left(\int \left| \frac{\rd (\mu'-\mu)}{\rd\theta}(\theta) \right| \rd\theta \right)^2, 
 \end{align*}
 where we used $\frac{\delta F_0}{\delta \mu}(\mu)(\theta) = \frac{1}{n} \sum_{i=1}^n \ell_i'( h_{i,\mu} ) h_i(\theta )$.
 
 Since $\hat{\mu}_t$ is the minimizer of the objective \eqref{eq:characterize-prox-gibbs} with $\mu=\mu^{(t)}$, i.e.,
\[ \hat{\mu}^{(t)} = \arg\min_{\mu'}\left\{ \int \frac{\delta F_0}{\delta \mu}(\mu^{(t)})(\theta) \rd \mu'(\theta) + \lambda \KL(\mu' \| \nu) \right\},\]
 we see for any $\xi \in \cP_2$ satisfying $\KL(\xi\| \nu) \leq r$, 
 \[ \int \frac{\delta F_0}{\delta \mu}(\mu^{(t)})(\theta) \rd \hat{\mu}^{(t)}(\theta) 
 \leq \int \frac{\delta F_0}{\delta \mu}(\mu^{(t)})(\theta) \rd \xi(\theta) + \lambda r.\]
 Therefore, applying this inequality with $\mu' = (1-\eta \gamma) \mu^{(t)} + \eta \gamma \hat{\mu}^{(t)}$ and $\mu = \mu^{(t)}$, and the convexity of $F_0$, 
\begin{align*}
 F_0&( \mu^{(t+1)} ) 
 - \eta \gamma \rho \\
 &\leq F_0( (1-\eta \gamma) \mu^{(t)} + \eta \gamma \hat{\mu}^{(t)} ) \\
 &\leq F_0(\mu^{(t)}) + \eta \gamma \int \frac{\delta F_0}{\delta \mu}(\mu^{(t)})(\theta) \rd (\hat{\mu}^{(t)} - \mu^{(t)})(\theta) 
 + \frac{\eta^2 \gamma^2 B^2L}{2n} \sum_{i=1}^n \left(\int \left| \frac{\rd (\hat{\mu}^{(t)}-\mu^{(t)})}{\rd\theta}(\theta) \right| \rd\theta \right)^2 \\
 &\leq F_0(\mu^{(t)}) + \eta \gamma \int \frac{\delta F_0}{\delta \mu}(\mu^{(t)})(\theta) \rd (\hat{\mu}^{(t)} - \mu^{(t)})(\theta) 
 + 2 \eta^2 \gamma^2 B^2L \\
 &\leq F_0(\mu^{(t)}) + \eta \gamma \int \frac{\delta F_0}{\delta \mu}(\mu^{(t)})(\theta) \rd (\xi - \mu^{(t)})(\theta) 
 + \eta \gamma \lambda r + 2 \eta^2 \gamma^2 B^2L \\
 &\leq F_0(\mu^{(t)}) + \eta \gamma (F_0(\xi) - F_0(\hat{\mu}^{(t)}))
 + \eta \gamma \lambda r + 2 \eta^2 \gamma^2 B^2L.
\end{align*}
This implies
\[ F_0(\mu^{(t+1)}) - F_0(\xi)
\leq \eta \gamma \rho + \eta \gamma \lambda r + 2 \eta^2 \gamma^2 B^2L
+ (1-\eta\gamma) (F_0(\mu^{(t)}) - F_0(\xi)). \]
With a slight modification, we get 
\[ F_0(\mu^{(t+1)}) - F_0(\xi) - \epsilon'
\leq (1-\eta\gamma) (F_0(\mu^{(t)}) - F_0(\xi) - \epsilon'). \]
where we set $\epsilon' = \rho + \lambda r + 2 \eta \gamma B^2 L$.
Recursively applying this inequality,
\[ F_0(\mu^{(T)}) - F_0(\xi) \leq \epsilon' + (1-\eta\gamma)^T F_0(\mu^{(0)}).\]
Because of arbitrariness of $\xi$, we conclude
\[ F_0(\mu^{(T)})  \leq \epsilon' + (1-\eta\gamma)^T F_0(\mu^{(0)}) + \inf_{\xi \in \bB_r} F_0(\xi).\]
\end{proof}

We explain that Theorem \ref{theorem:fw-convergence-appendix} can apply to Algorithm \ref{alg:implementable-efp} as discussed in Remark \ref{rem:fw-convergence}.
We redefine $\mu^{(t+1)}$ as $\mu^{(t+1)} = (1-\eta \gamma) \mu^{(t)} + \eta \gamma \nu^{(t)}$, where $\nu^{(t)}$ is an empirical distribution $\nu^{(t)}=\frac{1}{m}\sum_{r=1}^m \delta_{\theta_r^{(t)}}$ attained by Langevin Monte Carlo with the random initialization in Algorithm \ref{alg:implementable-efp}.
Moreover, we set $\overline{\nu}^{(t)}= \mathrm{Law}(\theta_r^{(t)})$.
Then, we can evaluate the gap between $\mu^{(t+1)}$ and an exact update $\mu'^{(t+1)} = (1-\eta \gamma) \mu^{(t)} + \eta \gamma \hat{\mu}^{(t)}$ as follows.
Suppose $\| h_i \|_\infty \leq B$ and $\ell_i$ is $C$-Lipschitz continuous. 
Then, we get
\begin{align*}
| F_0(\mu'^{(t+1)}) - F_0(\mu^{(t+1)}) |
&\leq \frac{C}{n}\sum_{i=1}^n \left| \bE_{\hat{\mu}'^{(t+1)}}[h_i] - \bE_{\mu^{(t+1)}}[h_i] | \right| \\
&= \frac{\eta \gamma C}{n}\sum_{i=1}^n \left| \bE_{\hat{\mu}^{(t)}}[h_i] - \bE_{\nu^{(t)}}[h_i] | \right| \\
&\leq \frac{\eta \gamma C}{n}\sum_{i=1}^n \left| \bE_{\hat{\mu}^{(t)}}[h_i] - \bE_{\overline{\nu}^{(t)}}[h_i] | \right| 
+ \frac{\eta \gamma C}{n}\sum_{i=1}^n \left| \bE_{\overline{\nu}^{(t)}}[h_i] - \bE_{\nu^{(t)}}[h_i] \right| \\
&\leq \eta \gamma B C \left\| \frac{\rd \hat{\mu}^{(t)}}{\rd \theta} - \frac{\rd \overline{\nu}^{(t)}}{\rd \theta}  \right\|_{L_1(\rd \theta)} 
+ \frac{\eta \gamma C}{n}\sum_{i=1}^n \left| \bE_{\overline{\nu}^{(t)}}[h_i] - \bE_{\nu^{(t)}}[h_i] \right| \\
&\leq \eta \gamma B C\sqrt{2\KL(\overline{\nu}^{(t)} \| \hat{\mu}^{(t)}) }
+ \frac{\eta \gamma C}{n}\sum_{i=1}^n | \bE_{\overline{\nu}^{(t)}}[h_i]  - \bE_{\nu^{(t)}}[h_i] |,
\end{align*}
where we applied Pinsker's inequality for the last inequality.

Therefore, we can estimate the particle and iteration complexities to satisfy the required precision (\ref{eq:fw-tolerance-appendix}) by applying the convergence rate of Langevin Monte Carlo \citep{vempala2019rapid} together with a standard concentration inequality as done in \citet{nitanda2020particle,oko2022psdca}. 
}

\end{document}